%% file: main.tex
\documentclass[9pt,sigconf,screen=true,bookmarks=false]{acmart}

\input{packages}
\input{macros}

\usepackage{geometry}

\copyrightyear{2022} 
\acmYear{2022} 
\setcopyright{rightsretained} 
\acmConference[DAC '22]{Proceedings of the 59th ACM/IEEE Design Automation Conference (DAC)}{July 10--14, 2022}{San Francisco, CA, USA}
\acmBooktitle{Proceedings of the 59th ACM/IEEE Design Automation Conference (DAC) (DAC '22), July 10--14, 2022, San Francisco, CA, USA}
\acmDOI{10.1145/3489517.3530400}
\acmISBN{978-1-4503-9142-9/22/07}

\begin{document}
\settopmatter{printacmref=false} 





\pagestyle{plain}

\title{
QuantumNAT: Quantum \underline{N}oise-\underline{A}ware \underline{T}raining with \\ Noise Injection, Quantization and Normalization
}

\author{$^1$Hanrui Wang, $^2$Jiaqi Gu, $^3$Yongshan Ding, $^4$Zirui Li, $^5$Frederic T. Chong, $^2$David Z. Pan, $^1$Song Han\\
\small{
$^1$Massachusetts Institute of Technology, $^2$University of Taxes at Austin, $^3$Yale University,
$^4$Rutgers University, $^5$University of Chicago
\texttt{\url{https://qmlsys.mit.edu}}}
}

\input{texts/0_abstract}

\maketitle

\input{texts/1_introduction}

\input{texts/2_background}
\input{texts/3_method}

\input{texts/4_evaluation}

\input{texts/5_conclusion}

\input{texts/7_acknowledgment}

\appendix
\input{texts/6_appendix}


{\small
\balance
\bibliographystyle{ACM-Reference-Format}
\bibliography{main}
}

\end{document}

%% file: packages.tex
\usepackage{color,xcolor}
\usepackage{epsfig}
\usepackage{graphicx}

\usepackage{adjustbox}
\usepackage{array}
\usepackage{booktabs}
\usepackage{colortbl}
\usepackage{float,wrapfig}
\usepackage{hhline}
\usepackage{multirow}
\usepackage{amsfonts}
\usepackage{mathtools}
\usepackage{bm}
\usepackage{nicefrac}
\usepackage{braket}

\usepackage{changepage}
\usepackage{extramarks}
\usepackage{fancyhdr}
\usepackage{setspace}
\usepackage{soul}
\usepackage{xspace}
\usepackage{textcomp}
\usepackage[subrefformat=parens,labelformat=parens]{subfig}

\usepackage{enumerate}
\usepackage{fancyhdr,graphicx,amsmath}
\usepackage[algo2e, ruled, vlined, noend]{algorithm2e}

\usepackage{optidef}
\usepackage{enumitem}

%% file: macros.tex
\definecolor{citecolor}{RGB}{34,139,34}
\definecolor{mydarkblue}{rgb}{0,0.08,1}
\definecolor{mydarkgreen}{rgb}{0.02,0.6,0.02}
\definecolor{mydarkred}{rgb}{0.8,0.02,0.02}
\definecolor{mydarkorange}{rgb}{0.40,0.2,0.02}
\definecolor{mypurple}{RGB}{111,0,255}
\definecolor{myred}{rgb}{1.0,0.0,0.0}
\definecolor{mygold}{rgb}{0.75,0.6,0.12}
\definecolor{myblue}{rgb}{0,0.2,0.8}
\definecolor{mydarkgray}{rgb}{0.,0.2,0.2}

\definecolor{lightred}{RGB}{255,235,235}
\definecolor{lightgreen}{RGB}{235,255,235}
\definecolor{lightblue}{RGB}{235,235,255}
\definecolor{lightcyan}{RGB}{235,255,255}
\definecolor{lightmagenta}{RGB}{255,235,255}
\definecolor{lightyellow}{RGB}{255,255,235}

\definecolor{qxkcolor}{RGB}{215,235,255}
\definecolor{softmaxcolor}{RGB}{230,235,255}
\definecolor{probxvcolor}{RGB}{255,255,235}

\definecolor{topkcolor}{RGB}{255,235,235}
\definecolor{zecolor}{RGB}{255,255,235}
\definecolor{dynacolor}{RGB}{235,255,255}

\definecolor{reviewcolor}{RGB}{0,0,200}

\renewcommand\footnotemark{}

\newcommand{\eg}{e.g., }

\newcommand{\name}{QuantumNAT\xspace}

\newcommand{\qnn}{QNN\xspace}
\newcommand{\pqc}{PQC\xspace}

\newcommand{\torchquantum}{TorchQuantum\xspace}

\newcommand{\x}{$\times$\xspace}

\newcommand{\E}{\mathbb{E}}
\newcommand{\Var}{\mathrm{Var}}

\newcommand{\nisq}{NISQ\xspace}

\newcommand{\cnot}{\texttt{CNOT}}
\newcommand{\swap}{\texttt{SWAP}}
\newcommand{\rxgate}{\texttt{RX}\xspace}
\newcommand{\rygate}{\texttt{RY}\xspace}
\newcommand{\rzgate}{\texttt{RZ}\xspace}
\newcommand{\rxyz}{\texttt{RXYZ}}
\newcommand{\uone}{\texttt{U1}}
\newcommand{\uthree}{\texttt{U3}\xspace}
\newcommand{\cuthree}{\texttt{CU3}\xspace}
\newcommand{\hgate}{\texttt{H}}
\newcommand{\tgate}{\texttt{T}}
\newcommand{\sgate}{\texttt{S}}
\newcommand{\xx}{\texttt{XX}}
\newcommand{\zx}{\texttt{ZX}}
\newcommand{\zz}{\texttt{ZZ}}
\newcommand{\cz}{\texttt{CZ}}
\newcommand{\cx}{\texttt{CNOT}\xspace}
\newcommand{\sqrtswap}{$\sqrt{\texttt{SWAP}}$}
\newcommand{\sqrth}{$\sqrt{\texttt{H}}$}
\newcommand{\xgate}{\texttt{X}\xspace}
\newcommand{\ygate}{\texttt{Y}\xspace}
\newcommand{\zgate}{\texttt{Z}\xspace}
\newcommand{\sx}{\texttt{SX}\xspace}
\newcommand{\idgate}{\texttt{ID}\xspace}

\newcounter{rlabelno}

\newcommand{\twoclassaccimprove}{42\%\xspace}
\newcommand{\fourclassaccimprove}{43\%\xspace}
\newcommand{\tenclassaccimprove}{23\%\xspace}

%% file: texts/0_abstract.tex
\begin{abstract}
Parameterized Quantum Circuits (PQC) are promising towards quantum advantage on near-term quantum hardware. However, due to the large quantum noises (errors), the performance of PQC models has a severe degradation on real quantum devices. Take Quantum Neural Network (\qnn) as an example, the accuracy gap between noise-free simulation and noisy results on IBMQ-Yorktown for MNIST-4 classification is over 60\%. Existing noise mitigation methods are \emph{general} ones without leveraging unique 
characteristics of \pqc; on the other hand, existing \pqc work \emph{does not} consider noise effect. To this end, we present \name, a \pqc-specific framework to perform noise-aware optimizations in both training and inference stages to improve robustness. 
We experimentally observe that the effect of quantum noise to \pqc measurement outcome is a linear map from noise-free outcome with a scaling and a shift factor.
Motivated by that, we propose \emph{post-measurement normalization} to mitigate the feature distribution differences between noise-free and noisy scenarios. Furthermore, to improve the robustness against noise, we propose \emph{noise injection} to the training process by inserting quantum error gates to \pqc according to realistic noise models of quantum hardware. Finally, \emph{post-measurement quantization} is introduced to quantize the measurement outcomes to discrete values, achieving the denoising effect. Extensive experiments on 8 classification tasks using 6 quantum devices demonstrate that \name improves accuracy by up to 43\%, and achieves over 94\% 2-class, 80\% 4-class, and 34\% 10-class classification accuracy measured on real quantum computers. The code for construction and noise-aware training of \pqc is available in the \href{https://github.com/mit-han-lab/torchquantum}{\torchquantum} library.
\end{abstract}

%% file: texts/1_introduction.tex
\section{Introduction}
\label{sec:introduction}
Quantum Computing (QC) is a new computational paradigm that can be exponentially faster than classical counterparts in various domains. Parameterized Quantum Circuits (PQC) are circuits containing trainable weights and are promising to achieve quantum advantages in current devices.
Among various PQCs, Quantum Neural Network (\qnn) is a popular algorithm in which a network of parameterized quantum gates are constructed and trained to embed data and perform certain ML tasks on a quantum computer, similar to the training and inference of classical neural networks.

Currently we are in the Noisy Intermediate Scale Quantum (\nisq) stage, in which quantum operations suffer from a high error rate of $10^{-2}$ to $10^{-3}$, much higher than CPUs/GPUs ($10^{-6}$ FIT). The quantum errors unfortunately introduces detrimental influence on \pqc accuracy. Figure~\ref{fig:motivation} shows the single-qubit gate error rates and the measured accuracy of classification tasks with \qnn on different hardware. Three key observations are: (1) Quantum error rates ($10^{-3}$) are much larger than classical CMOS devices' error rates ($10^{-6}$ failure per $10^{9}$ device hours). (2) Accuracy on real hardware is significantly degraded (up to 64\%) compared with noise-free simulation. (3) The same \qnn on different hardware has distinct accuracy due to different gate error rates. IBMQ-Yorktown has a five times larger error rate than IBMQ-Santiago, and higher error causes lower accuracy.

Researchers have proposed noise mitigation techniques~\citep{temme2017error, wille2019mapping} to reduce the noise impact. However, they are \emph{general methods} without considering the unique characteristics of \pqc, and can only be applied to \pqc inference stage. On the other hand, existing \pqc work~\citep{farhi2018classification, jiang2021co} does not consider the noise impact. This paper proposes a \pqc-specific noise mitigation framework called \name that optimizes \pqc robustness in \emph{both training and inference} stages, boosts the \emph{intrinsic robustness} of \pqc parameters, and improves accuracy on \emph{real quantum machines}.

\input{figtex/fig_motivation}

\name comprises a three-stage pipeline. The first step, \emph{post-measurement} \emph{normalization} normalizes the measurement outcomes on each quantum bit (qubit) across data samples, thus removing the quantum error-induced distribution shift. Furthermore, we inject noise to the \pqc training process by performing \emph{error gate insertion}. The error gate types and probabilities are obtained from hardware-specific realistic quantum noise models provided by QC vendors. During training, we iteratively sample error gates, insert them to \pqc, and updates weights. Finally, \emph{post-measurement quantization} is further proposed to reduce the precision of measurement outcomes from each qubit and achieve a denoising effect.

\input{figtex/fig_circuits}

In this paper, we are mainly using \textit{QNNs as benchmarks} but the techniques can also be applied to \textit{other PQCs}. The contributions of \name are as follows:

\begin{itemize}[leftmargin=*]
    \item A systematic pipeline to mitigate noise impact: post-measurement normalization noise injection and post-measurement quantization.
    \item Extensive experiments on 8 ML tasks with 5 different design spaces on 6 quantum devices show that \name can improve accuracy by up to \twoclassaccimprove, \fourclassaccimprove, \tenclassaccimprove for 2-class, 4-class and 10-class classification tasks and demonstrates over 94\%, 80\% and 34\% accuracy for 2-, 4-, and 10-classifications on real quantum hardware. 
    \item The code for construction and noise-aware training of \pqc is available at the \href{https://github.com/mit-han-lab/torchquantum}{\torchquantum} library. It is an convenient infrastructure to query noise models from QC providers such as IBMQ, extract noise information, perform training on CPU/GPU and finally deploy on real QC.
\end{itemize}

%% file: figtex/fig_motivation.tex
\begin{figure}[t]
    \centering
    \includegraphics[width=\columnwidth]{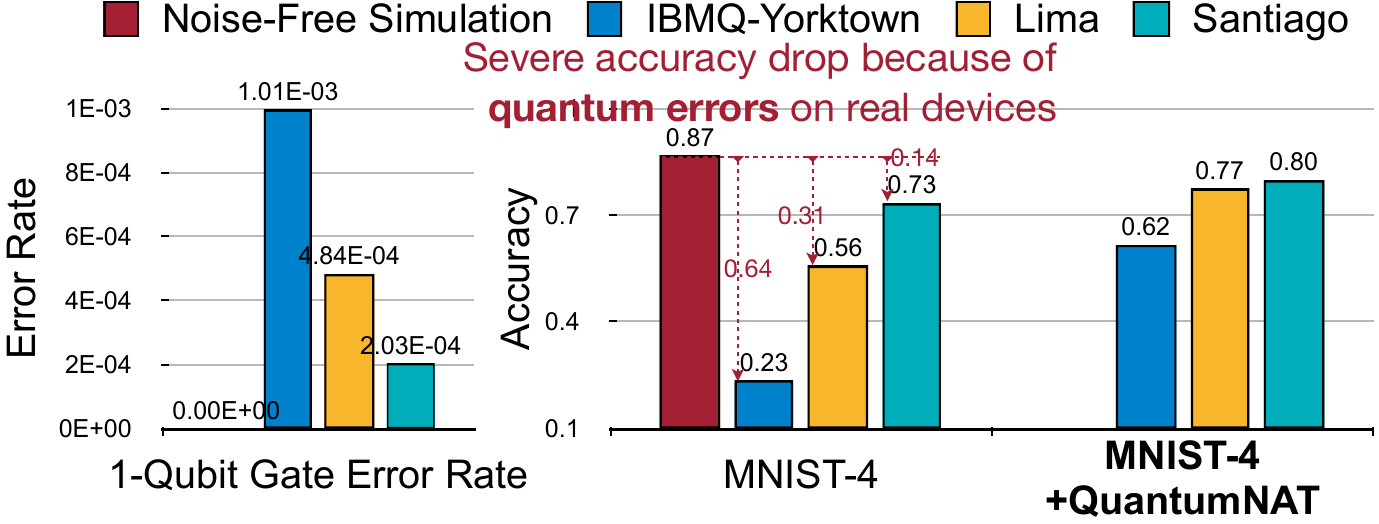}
    \vspace{-10pt}
    \caption{
    \emph{Left:} Current quantum hardware has much larger error rates (around $10^{-3}$) than classical CPUs/GPUs. \emph{Right:} Due to the errors, \pqc (\qnn) models suffer from severe accuracy drops. Different devices have various error magnitudes, leading to distinct accuracy. These motivate \name, a \emph{hardware-specific noise-aware} \pqc training approach to improve robustness and accuracy.
    }
    \vspace{-10pt}
    \label{fig:motivation}
\end{figure}

%% file: figtex/fig_circuits.tex
\begin{figure*}[t]
    \centering
    \includegraphics[width=0.9\textwidth]{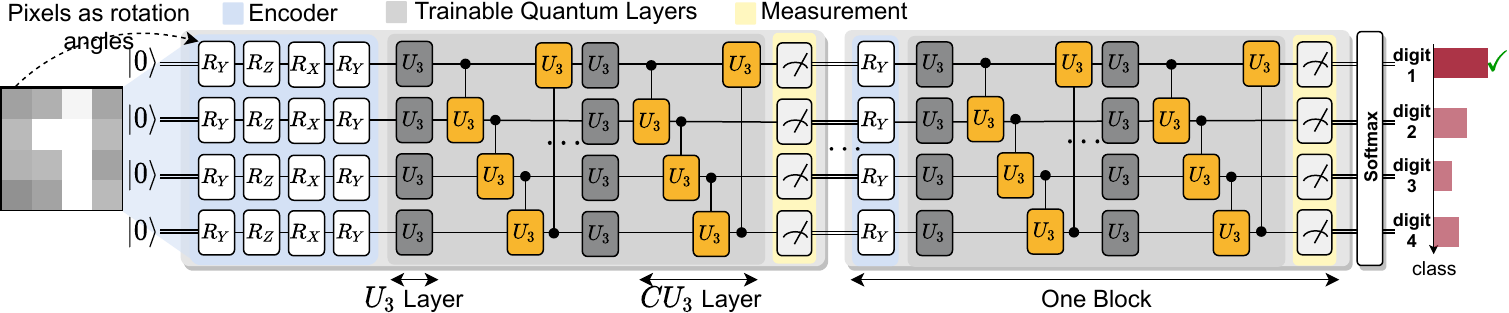}
    \vspace{-8pt}
    \caption{Quantum Neural Networks Architecture. \qnn has multiple blocks, each has an encoder to encode classical values to quantum domain, quantum layers with trainable weights, and a measurement layer that obtains classical values.}
    \vspace{-8pt}
    \label{fig:circuits}
\end{figure*}

%% file: texts/2_background.tex
\section{Background and Related Work}
\label{sec:background}

\noindent\textbf{QML and QNN.}
Quantum machine learning explores performing ML tasks on quantum devices.
The path to \emph{quantum advantage} on QML is typically provided by the quantum circuit's ability to generate and estimate highly complex kernels, which would otherwise be intractable to compute with conventional computers.
They have been shown to have potential speed-up over classical counterparts in various tasks, including metric learning, data analysis, and principal component analysis. Quantum Neural Networks is one type of QML models using variational quantum circuits with trainable parameters to accomplish feature encoding of input data and perform complex-valued linear transformations thereafter. Various theoretical formulations for \qnn have been proposed such as quantum Boltzmann machine~\cite{amin2018quantum} and quantum classifier~\citep{farhi2018classification, wang2021quantumnas, wang2021exploration, wang2022onchipqnn}, etc.  

\noindent\textbf{Quantum error mitigation.}
As the error forms the bottleneck of the quantum area. Researchers have developed various error mitigation techniques~\cite{wille2019mapping}. Extrapolation methods~\citep{temme2017error} perform multiple measurements of a quantum circuit under different error rates and then extrapolate the ideal measurement outcomes when there is no noise.~\cite{liang2021can} trains \pqc using RL with noisy simulator.
\name is fundamentally different from existing methods: (i) Prior work focuses on low-level numerical correction in inference only; \name embraces more optimization freedom in both \emph{training and inference}. It improves the intrinsic robustness and statistical fidelity of \emph{\pqc parameters}. (ii) \pqc has a good built-in error-tolerance which motivates \name's post-measurement quantization to reduce the numerical precision of intermediate results while preserving accuracy. (iii) \name has a small overhead ($<$2\%), while others introduce high measurements, circuit complexity cost, etc. We show that existing extrapolation method is orthogonal to \name in Section~\ref{sec:results}.

\noindent\textbf{Quantization and noise injection of classical NN.}
To improve NN efficiency, extensive work has been explored to trim down redundant bit representation in NN weights and activations~\citep{han2015deep, wang2021spatten}.
Though low-precision quantization limits the model capacity, it can improve the generalization and robustness~\citep{lin2019defensive}.
An intuitive explanation is that quantization corrects errors by value clamping, thus avoiding cascaded error accumulation.
Moreover, by sparsifying the parameter space, quantization reduces the NN complexity as a regularization mechanism that mitigates potential overfitting issues.
Similarly, injecting noises into neural network training is demonstrated to help obtain a smoothed loss landscape for better generalization~\citep{155944}.

%% file: texts/3_method.tex
\section{Noise-Aware \pqc Training}
\label{sec:method}

We use \qnn as the benchmark \pqc in this work. Figure~\ref{fig:circuits} shows the \qnn architecture. The inputs are classical data such as image pixels, and the outputs are classification results. The \qnn consists of multiple blocks. Each has three components: encoder encodes the classical values to quantum states with rotation gates such as \rygate; trainable quantum layers contain parameterized gates that can be trained to perform certain ML tasks; measurement part measures each qubit and obtains a classical value. The measurement outcomes of one block are passed to the next block. For the MNIST-4 example in Figure~\ref{fig:circuits}, the first encoder takes the pixels of the down-sampled 4\x4 image as rotation angles $\theta$ of 16 rotation gates. The measurement results of the last block are passed through a Softmax to output classification probabilities. \name overview is in Figure~\ref{fig:flowchart}. 

\input{figtex/fig_flowchart}

\subsection{Post-Measurement Normalization}
\label{sec:normalization}
\noindent\textbf{Measurement outcome shift due to quantum noises.}
Before delving into the noise mitigation techniques, we first show analytically how quantum noises influence the \qnn block output. 
The measurement outcomes of the \qnn are sensitive to both the input parameters and any perturbations by some noisy quantum process. 
This section provides insights on such noisy transformations and discusses their impacts on \qnn inference. 

\begin{theorem} \textnormal{(informal version)}\textbf{.} \label{th:linear_noise}
The measurement outcome $y$ of a quantum neural network for the training input data $x$ is transformed by the quantum noise that the system undergoes with a linear map $f(y_x) = \gamma y_x + \beta_x$, where the translation $\beta_x$ depends on the input $x$ and quantum noises, while scaling factor $\gamma$ is input independent.
\end{theorem}

We refer to Appendix Section \ref{proof:linear_noise} for background and a complete proof. The main theoretical contribution of this theorem equips our proposed normalization methodology with robustness guarantees. Most importantly, we observe that the changes in measurement results can often be compensated by proper post-measurement normalization across input batches. For simplicity, we restrict our analysis on $Z$-basis single-qubit measurement outcome $y$. Similar analytical results for multi-qubit general-basis measurement will follow if we apply the same analysis qubit by qubit. Theorem \ref{th:linear_noise} is most powerful when applied on a small batch of input data $\mathbf{x} = \{x_1, \dotsc, x_m\}$ where each $x_i$ is a set of classical input values for the encoder of the QNN and $m$ is the size of the batch. In an ideal noiseless scenario, the QNN model outputs measurement result $y_i$ for each input $x_i$. For a noisy QNN, the measurement result undergoes a composition of two transformations: (1) a constant scaling by $\gamma$; (2) a input-specific shift by $\beta_i$, i.e., $f(y_i) = \gamma y_i + \beta_i$. In the realistic noise regime, the scaling constant $\gamma \in[-1,1]$. However, for small noises, $\gamma$ is close to 1, and $\beta_i$ is close to 0. Therefore, the distribution of noisy measurement results undergoes a constant scaling by $\gamma \leq 1$ and a small shift by each $\beta_i$. In the small-batch regime when $\boldsymbol{\beta} = \{\beta_1, \dotsc, \beta_m\}$ has small variance, the distribution is shifted by its mean $\beta=\E[\boldsymbol{\beta}]$. Thus $f(y_i) \approx \gamma y_i + \beta$.

\noindent\textbf{Post-measurement normalization}.
Based on the analysis above, we propose \emph{post-measurement normalization} to offset the distribution scaling and shift. 
For each qubit, we collect its measurement results on a batch of input samples, compute their mean and std., then make the distribution of each qubit across the batch \emph{zero-centered} and of \emph{unit variance}. This is performed during both training and inference. 
During training, for a batch of measurement results: $\boldsymbol{y} = \{y_1, \dotsc, y_m\}$, the normalized results are $\hat{ y_{i}}=(y_{i}-\E[\boldsymbol{y}])/\sqrt{\Var(\boldsymbol{y})}$. 
For noisy inference, we correct the error as $\widehat{f(y_{i})}=(f(y_{i})-\E[f(\boldsymbol{y})])/\sqrt{\Var(f(\boldsymbol{y}))} = ((\gamma y_{i} + \beta) - (\gamma \E[\boldsymbol{y}] + \beta)) / \sqrt{\gamma^2\Var(\boldsymbol{y})} = \hat{y_{i}}$.

Figure~\ref{fig:norm} compares the noise-free measurement result distribution of 4 qubits (blue) with their noisy counterparts (yellow) for MNIST-4. Qualitatively, we can clearly observe that the post-measurement normalization reduces the mismatch between two distributions. 
Quantitatively, we adopt signal-to-noise ratio, $SNR = \|\bm{A}\|_{2}^2/\|\bm{A}-\widetilde{\bm{A}}\|_{2}^2$, the inverse of relative matrix distance (RMD), as the metric. The SNR on each qubit and each individual measurement outcome is clearly improved. Though similar, it is different from Batch Normalization~\citep{ioffe2015batch} as the testing batch uses its own statistics instead of that from training, and there is no trainable affine parameter.

\input{figtex/fig_norm}

\subsection{Quantum Noise Injection}
\label{sec:error_injection}

Although the normalization above mitigates error impacts, we can still observe small discrepancies on each individual measurement outcome, which degrade the accuracy. Therefore, to make the \qnn model robust to those errors, we propose noise injection to the training process. 

\input{figtex/fig_gateinsertion}
\noindent\textbf{Quantum error gate insertion.} As introduced in Section~\ref{sec:background}, different quantum errors can be approximated by Pauli errors via Pauli Twirling. The effect of Pauli errors is the random insertion of Pauli \xgate, \ygate, and \zgate gates to the model with a probability distribution $\mathcal{E}$. How to compute $\mathcal{E}$ is out of the scope of this work. 
But fortunately, we can directly obtain it from the realistic device noise model provided by quantum hardware manufacturers such as IBMQ. The noise model specifies the probability $\mathcal{E}$ for different gates on each qubit. For single-qubit gates, the error gates are inserted \emph{after} the original gate. For two-qubit gates, error gates are inserted after the gate on \emph{one or both} qubits. For example, the \sx gate on qubit 1 on IBMQ-Yorktown device has $\mathcal{E}$ as \{\xgate: 0.00096, \ygate: 0.00096, \zgate: 0.00096, None: 0.99712\}. When `None' is sampled, we will not insert any gate. The same gate on different qubits or different hardware will have up 10\x probability difference.
As in Figure~\ref{fig:gateinsertion}, during training, for each \qnn gate, we sample error gates based on $\mathcal{E}$ and insert it after the original gate. A new set of error gates is sampled for each training step. In reality, the \qnn is compiled to the basis gate set of the quantum hardware (\eg \xgate, \cx, \rzgate, \cx, and \idgate) before performing gate insertion and training. We will also scale the probability distribution by a constant \emph{noise factor} $T$ and scale the \xgate, \ygate, \zgate probability by $T$ during sampling. $T$ factor explores the trade-off between adequate noise injection and training stability. Typical $T$ values are in the range of [0.5, 1.5]. The gate insertion overhead is typically less than 2\%.

\noindent\textbf{Readout noise injection.}
Obtaining classical values from qubits is referred as readout/measurement, which is also error-prone. The realistic noise model provides the statistical readout error in the form of a 2\x2 matrix for each qubit. For example, the qubit 0 of IBMQ-Santiago has readout error matrix [[0.984, 0.016], [0.022, 0.978]] which means the probability of measuring a $\ket{0}$ as 0 is 0.984 and as 1 is 0.016. We emulate the readout error effect during training by changing the measurement outcome. For instance, originally $P(0)=0.3, P(1)=0.7$, the noise injected version will be $P'(0) = 0.3 \times 0.984 + 0.7 \times 0.022=0.31, P'(1) = 0.7 \times 0.978 + 0.3 \times 0.016=0.69$.

\noindent\textbf{Direct perturbation.}
Besides gate insertion, we also experimented with directly perturbing measurement outcomes or rotation angles as noise sources. For outcome perturbation, with benchmarking samples from the validation set, we obtain the error $Err$ distribution between the noise-free and noisy measurement results and compute the mean $\mu_{Err}$ and std $\sigma_{Err}$. During training, we directly add noise with Gaussian distribution $\mathcal{N}(\mu_{Err}, \sigma_{Err}^2)$ to the normalized measurement outcomes. Similarly, for rotation angle perturbation, we add Gaussian noise to the angles of all rotation gates in \qnn and make the effect of rotation angle Gaussian noise on measurement outcomes similar to real QC noise.  We show in Section~\ref{sec:results} that the gate insertion method is better than direct perturbations.

\subsection{Post-Measurement Quantization}
\label{sec:quantization}
\input{figtex/fig_measurement_quantization}
Finally, we propose post-measurement quantization on the normalized results to further denoise the measurement outcomes. We first clip the outcomes to $[p_{min}, p_{max}]$, where $p$ are pre-defined thresholds, and then perform uniform quantization. The quantized values are later passed to the next block's encoder. Figure~\ref{fig:measurement_quantization} shows one real example from Fashion-4 on IBMQ-Santiago with five quantization levels and $p_{min}=-2, p_{max}=2$. The left/middle matrices show the error maps between noise-free and noisy outcomes before/after quantization. Most errors can be corrected back to zero with few exceptions of being quantized to a wrong centroid. The MSE is reduced from 0.235 to 0.167, and the SNR is increased from 4.256 to 6.455. We also add a loss term $||y-Q(y)||_2^{2}$ to the training loss, as shown on the right side, to encourage outcomes to be near to the quantization centroids to improve error tolerance and reduce the chance of being quantized to a wrong centroid. Besides improving robustness, quantization also reduces the control complexity of rotation gates.

%% file: figtex/fig_flowchart.tex
\begin{figure*}[t]
    \centering
    \includegraphics[width=\textwidth]{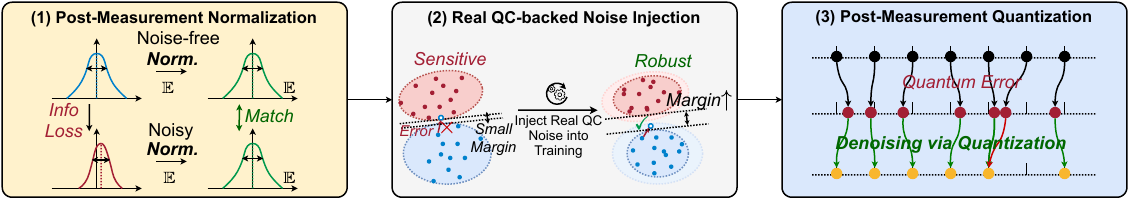}
    \caption{\name Overview. (1) Post-measurement normalization matches the distribution of measurement results between noise-free simulation and real QC. (2) Based on realistic noise models, noise-injection inserts \emph{quantum error gates} to the training process to increase the classification margin between classes. (3) Measurement outcomes are further quantized for denoising.}
    \label{fig:flowchart}
\end{figure*}

%% file: figtex/fig_norm.tex
\begin{figure}[t]
    \centering
    \includegraphics[width=\columnwidth]{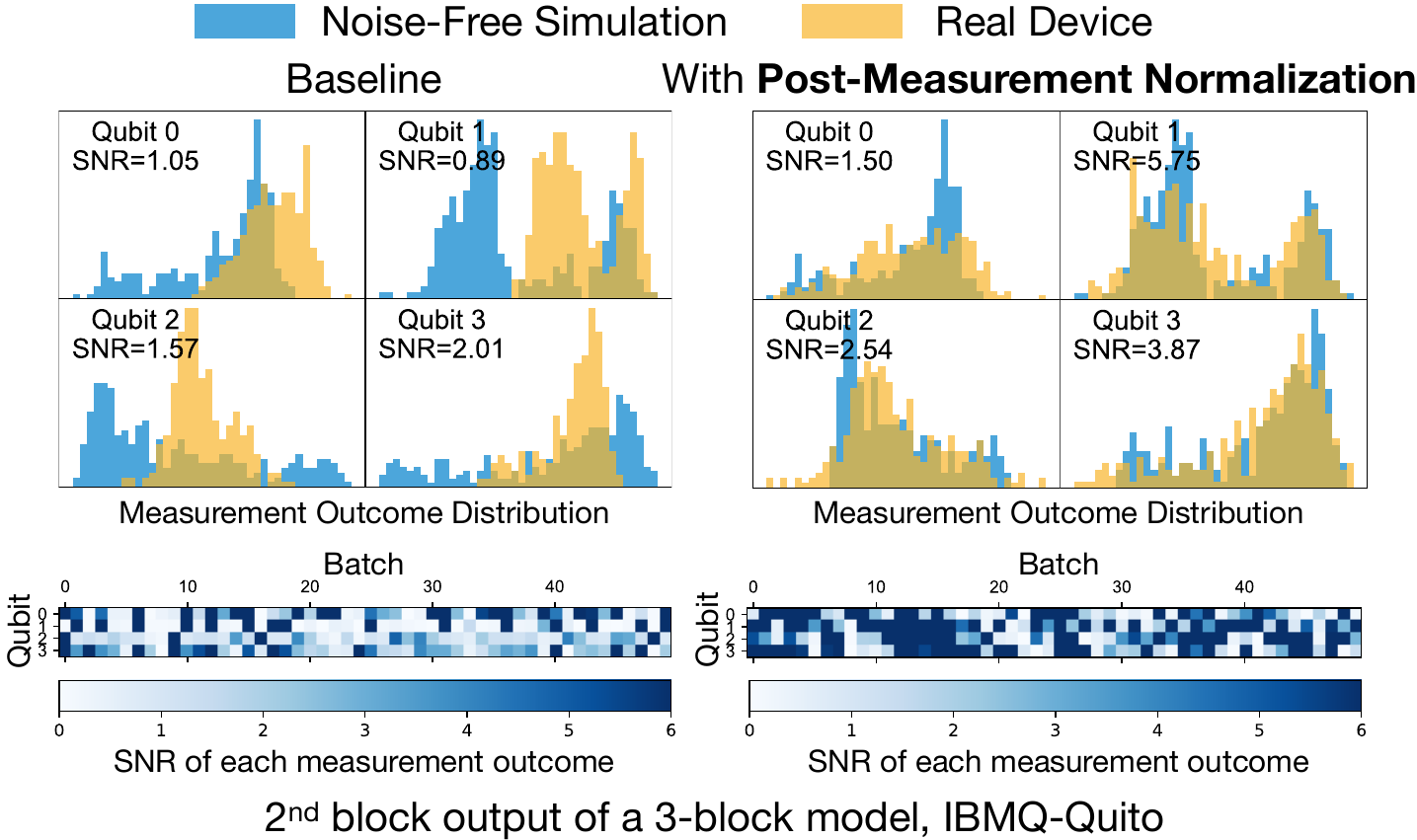}
    \caption{Post-measurement normalization reduces the distribution mismatch between noise-free simulation and noisy results, thus improving the Signal-to-Noise Ratio (SNR).}
    \vspace{-10pt}
    \label{fig:norm}
\end{figure}

%% file: figtex/fig_gateinsertion.tex

\begin{figure*}[t]
    \centering
    \includegraphics[width=\textwidth]{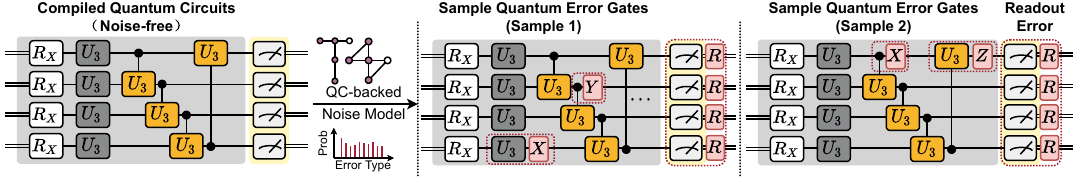}
    \vspace{-20pt}
    \caption{Noise injection via error gate insertion. \xgate, \ygate, \zgate are sampled Pauli error gates. \texttt{R} is the injected readout error. Probabilities for gate insertion are obtained from real device noise models.}
    \label{fig:gateinsertion}
    \vspace{-10pt}
\end{figure*}

%% file: figtex/fig_measurement_quantization.tex
\begin{figure}[t]
    \centering
    \includegraphics[width=\columnwidth]{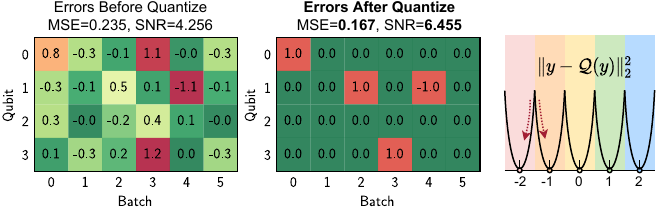}
    \caption{\emph{Left}: Error maps before and after post-measurement quantization. Most errors can be corrected.
    \emph{Right}: 5-level quantization buckets with a quadratic penalty loss.
    }
    \label{fig:measurement_quantization}
    \vspace{-15pt}
\end{figure}

%% file: texts/4_evaluation.tex
\section{Experiments}
\label{sec:results}
\subsection{Experiment Setups}
\label{sec:experiment_setups}
\noindent\textbf{Datasets.} We conduct experiments on 8 classification tasks including MNIST~\citep{726791} 10-class, 4-class (\texttt{0, 1, 2, 3}). and 2-class (\texttt{3, 6}); Vowel 4-class (\texttt{hid, hId, had, hOd}); Fashion~\citep{xiao2017fashion} 10-class, 4-class (\texttt{t-shirt/top, trouser, pullover, dress}), and 2-class (\texttt{dress, shirt}), and CIFAR~\citep{cifar10} 2-class (\texttt{frog, ship}). MNIST, Fashion, and CIFAR use 95\% images in `train' split as training set and 5\% as the validation set. Due to the limited real QC resources, we use the first 300 images of `test' split as test set. Vowel-4 dataset (990 samples) is separated to train:validation:test = 6:1:3 and test with the whole test set. MNIST and Fashion images are center-cropped to $24\times24$; and then down-sample to 4$\times$4 for 2- and 4-class, and 6$\times$6 for 10-class; CIFAR images are converted to grayscale, center-cropped to 28\x28, and down-sampled to 4\x4. All down-samplings are performed with average pooling. For vowel-4, we perform feature principal component analysis (PCA) and take 10 most significant dimensions.

\noindent\textbf{\qnn models.} 
\qnn models for 2 and 4-class use 4 qubits; 10-class uses 10.
The first quantum block's encoder embeds images and vowel features.
For $4\times4$ images, we use 4 qubits and 4 layers with 4 \rygate, 4 \rxgate, 4 \rzgate, and 4 \rygate~gates in each layer, respectively.
There are in total 16 gates to encode the 16 classical values as the rotation angles.
For $6\times6$ images, 10 qubits and 4 layers are used with 10 \rygate, 10 \rxgate, 10 \rzgate, and 6 \rygate~gates in each layer, respectively.
10 vowel features, uses 4 qubits and 3 layers with 4 \rygate, 4 \rxgate,~and 2 \rzgate~gates on each layer for encoding.
For trainable quantum layers, we use \uthree and \cuthree layers interleaved as in Figure~\ref{fig:circuits} except for Table~\ref{tab:design_space}. 
For measurement, we measure the expectation values on Pauli-\texttt{Z} basis and obtain a value [-1, 1] from each qubit. The measurement outcome goes through post-measurement normalization and quantization and is used as rotation angles for \rygate gates in the next block's encoder.
After the last block, for two-classifications, we sum the qubit 0 and 1, 2 and 3 measurement outcomes, respectively, and use Softmax to get probabilities. 
For 4 and 10-class, Softmax is directly applied to measurement outcomes. 


\noindent\textbf{Quantum hardware and compiler configurations.}.
\label{sec:compute_resources}
We use IBMQ quantum computers via Qiskit~\citep{ibmq} APIs. We study 6 devices, with \#qubits from 5 to 15 and Quantum Volume from 8 to 32. We also employ Qiskit for compilation. The optimization level is set to 2 for all experiments. All experiments run 8192 shots. The noise models we used are off-the-shelf ones updated by IBMQ team.

\subsection{Main Results}

\noindent\textbf{QNN results.} We experiment with four different \qnn architectures on 8 tasks running on 5 quantum devices to demonstrate \name's effectiveness. For each benchmark, we experiment with noise factor $T=\{0.1, 0.5, 1, 1.5\}$ and quantization level among \{3, 4, 5, 6\} and select one out of 16 combinations with the lowest loss on the \emph{validation set} and test on the \emph{test set}. 
Normalization and quantization are not applied to the last block's measurement outcomes as they are directly used for classification. 
As in Table~\ref{tab:main}, \name consistently achieves the highest accuracy on 26 benchmarks. The third bars of Athens are unavailable due to its retirement. On average, normalization, noise injection and quantization improve accuracy by 10\%, 9\%, and 3\%, respectively. A larger model does not necessarily have higher accuracy. For example, Athens' model is 7.5\x larger than Yorktown with higher noise-free accuracy. However, due to more errors introduced by the larger model, the real accuracy is lower.

\input{tables/tab_main}
\input{tables/tab_design_space}
\input{tables/tab_scalability}

\input{tables/tab_ortho}
\input{tables/tab_readout_norm}

\noindent\textbf{Performance on different design spaces.}
In Table~\ref{tab:design_space}, we evaluate \name on different \qnn design spaces. Specifically, the trainable quantum layers in one block
of `\zz+\rygate'~\citep{lloyd2020quantum} space contains one layer of \zz~gate, with ring connections, and one \rygate~layer. 
`\rxyz'~\citep{mcclean2018barren} space has five layers: \sqrth, \rxgate, \rygate, \rzgate, and \cz. `\zx+\xx'~\citep{farhi2018classification} space has two layers: \zx~and \xx. 
`\rxyz+\uone+\cuthree'~\citep{henderson2020quanvolutional} space, according to their random circuit basis gate set, has 11 layers in the order of \rxgate, \sgate, \cnot, \rygate, \tgate, \swap, \rzgate, \hgate, \sqrtswap, \uone~and \cuthree. We conduct experiments on MNIST-4 and Fashion-2 on 2 devices. In 13 settings out of 16, \name achieves better accuracy. Thus, \name is a general technique agnostic to \qnn model size and design space.

\noindent\textbf{Scalability.}
When classical simulation is infeasible, we can move the
the noise-injected training to real QC using techniques such as \emph{parameter shift}~\citep{crooks2019gradients}. In this case, the training cost is \emph{linearly} scaled with qubit number. Post-measurement normalization and quantization are also \emph{linearly} scalable because they are performed on the measurement outcomes. Gradients obtained with real QC are naturally noise-aware because they are directly influenced by quantum noise. To demonstrate the practicality, we train a 2-class task with two numbers as input features~\citep{jiang2021co} (Table~\ref{tab:scalability}). The \qnn has 2 blocks; each with 2 RY and a CNOT gates. The noise-unaware baseline trains the model on classical part and test on real QC. In \name, we train the model with parameter shift and test, both on real QC. We consistently outperform noise-unaware baselines.

\input{figtex/fig_noise_model_ablation}

\input{figtex/fig_noiselevelcontour}

\input{figtex/fig_breakdown}
\noindent\textbf{Compatibility with existing noise mitigation.} \name is orthogonal to existing noise mitigation such as extrapolation method. It can be combined with post-measurement normalization (Table~\ref{tab:ortho}). The QNN model has 2 blocks, each with three U3+CU3 layers. For ``Normalization only", the measurement outcomes of the 3-layer block are normalized across the batch dimension. For ``Extrapolation + Normalization", we use extrapolation to estimate the standard deviation of noise-free measurement outcomes. We firstly train the QNN model to convergence and then repeat the 3 layers to 6, 9, 12 layers and obtain four standard deviations of measurement outcomes. Then we linearly extrapolate them to obtain noise-free std. We normalize the measurement outcomes of the 3-layer block to make their std the same as noise-free and then apply the proposed post-measurement norm. Results show that the extrapolation can further improve the \qnn accuracy thus being orthogonal.

\subsection{Ablation Studies}
\noindent\textbf{Ablation on post-measurement normalization.}
Table~\ref{tab:norm} compares the accuracy and signal-to-noise ratio (SNR) before and after post-measurement normalization on MNIST-4. We study 4 different \qnn architectures and evaluate on 3 devices. The normalization can significantly and consistently increase SNR.

\noindent\textbf{Ablation on different noise injection methods.}
Figure~\ref{fig:noisemodel} compares different noise injection methods. Gaussian noise statistics for perturbations are obtained from error benchmarking. The left side shows accuracy without quantization. 
With different noise factors $T$, the gate insertion and measurement outcome perturbation have similar accuracy, both better than rotation angle perturbation. A possible explanation is that the rotation angle perturbation does not consider non-rotation gates such as \xgate and \sx.
The right side further investigates the first two methods' performance with quantization. We set noise factor $T=0.5$ and 
alter quantization levels. Gate insertion outperforms perturbation by 11\% on average on 3 different devices and \qnn models. The reason is: directly added perturbation on measurement outcomes can be easily canceled by quantization, and thus it is harder for noise injection to take effect.

\noindent\textbf{Noise factor and post-measurement quantization level analysis.}
We visualize the \qnn accuracy contours on Fashion-4 on IBMQ-Athens with different noise factors and quantization levels in Figure~\ref{fig:noiselevelcontour} left. The best accuracy occurs for factor 0.2 and 5 levels. Horizontal-wise, the accuracy first goes up and then goes down. This is because too few quantization levels hurt the \qnn model capacity; too many levels cannot bring sufficient denoising effect. Vertical-wise, the accuracy also goes up and then down. Reason: when the noise is too small, the noise-injection effect is weak, thus cannot improve the model robustness; while too large noise makes the training process unstable and hurts accuracy.

\noindent\textbf{Visualization of QNN extracted features.} MNIST-2 classification result is determined by which feature is larger between the two: feature one is the sum of measurement outcomes of qubit 0 and 1; feature 2 is that of qubit 2 and 3. We visualize the two features obtained from experiments on Belem in a 2-D plane as in Figure~\ref{fig:noiselevelcontour} right. The blue dash line is the classification boundary. The circles/stars are samples of digit `3' and `6'. All the baseline points (yellow) huddled together, and all digit `3' samples are misclassified. With normalization (green), the distribution is significantly expanded, and the majority of `3' is correctly classified. Finally, after noise injection (red), the margin between the two classes is further enlarged, and the samples are farther away from the classification boundary, thus becoming more robust. 

\noindent\textbf{Breakdown of accuracy gain.}
Figure~\ref{fig:breakdown} shows the performance of only applying noise-injection, only applying quantization, and both. Using two techniques individually can both improve accuracy by 9\%. Combining two techniques delivers better performance with a 17\% accuracy gain. This indicates the benefits of synergistically applying three techniques in \name.

%% file: tables/tab_main.tex
\begin{table}[t]
\centering
\renewcommand*{\arraystretch}{0.8}
\setlength{\tabcolsep}{1pt}

\caption{\name consistently achieves the highest accuracy, with on average 22\% better. `B' for Block, `L' for Layer. }
\label{tab:main}
\resizebox{1.08\columnwidth}{!}{%
\begin{tabular}{llcccccc}
\toprule
Model & Method & MNIST-4 & Fash.-4 & Vow.-4 & MNIST-2 & Fash.-2 & Cifar-2 \\
\midrule
\multirow{4}{*}{\shortstack[l]{2B$\times$12L\\Santiago}} & Baseline & 0.30 & 0.32 & 0.28 & 0.84 & 0.78 & 0.51 \\ 
& + Post Norm. & 0.41 & 0.61 & 0.29 & 0.87 & 0.68 & 0.56 \\
& + Gate Insert. & 0.61 & 0.70 & 0.44 & 0.93 & 0.86 & 0.57 \\
& + Post Quant. & \textbf{0.68} & \textbf{0.75} & \textbf{0.48} & \textbf{0.94} & \textbf{0.88} & \textbf{0.59} \\
\midrule
\multirow{4}{*}{\shortstack[l]{2B$\times$2L\\Yorktown}} & Baseline & 0.43 & 0.56 & 0.25 & 0.68 & 0.70 & 0.52 \\ 
& + Post Norm. & 0.57 & 0.60 & 0.38 & 0.86 & 0.72 & 0.56 \\ 
& + Gate Insert. & 0.58 & 0.60 & \textbf{0.45} & 0.91 & 0.85 & 0.57 \\ 
& + Post Quant. & \textbf{0.62} & \textbf{0.65} & 0.44 & \textbf{0.93} & \textbf{0.86} & \textbf{0.60} \\ 
\midrule
\multirow{4}{*}{\shortstack[l]{2B$\times$6L\\Belem}} & Baseline & 0.28 & 0.26 & 0.20 & 0.46 & 0.52 & 0.50 \\ 
& + Post Norm. & 0.52 & 0.57 & 0.33 & 0.81 & 0.62 & 0.51 \\ 
& + Gate Insert.  & 0.52 & 0.60 & 0.37 & 0.84 & \textbf{0.82} & 0.57 \\ 
& + Post Quant. & \textbf{0.58} & \textbf{0.62} & \textbf{0.41} & \textbf{0.88} & 0.80 & \textbf{0.61} \\ 
\midrule
\multirow{4}{*}{\shortstack[l]{3B$\times$10L\\Athens}} & Baseline & 0.29 & 0.36 & 0.21 & 0.54 & 0.46 & 0.49 \\ 
& + Post Norm. & 0.44 & 0.46 & 0.37 & 0.51 & 0.51 & 0.50 \\ 
& + Gate Insert.  & - & - & - & - & - & - \\ 
& + Post Quant. & \textbf{0.56} & \textbf{0.64} & \textbf{0.41} & \textbf{0.87} & \textbf{0.64} & \textbf{0.53} \\ 
\midrule
\midrule
Model & Method & MNIST-10 & Fash.-10 & Avg.-All \\
\midrule
\multirow{4}{*}{\shortstack[l]{2B$\times$2L\\Melbo.}} & Baseline & 0.11 & 0.09 & 0.42 \\
& + Post Norm. & 0.08 & 0.12 & 0.52 \\
& + Gate Insert.  & 0.25 & 0.24 &0.61 \\
& + Post Quant. & \textbf{0.34} & \textbf{0.31} & \textbf{0.64} \\
\bottomrule
\end{tabular}%
}
\end{table}

%% file: tables/tab_design_space.tex
\begin{table}[t]
\centering
\renewcommand*{\arraystretch}{0.8}
\setlength{\tabcolsep}{5pt}
\caption{Accuracy on different design spaces.}
\vspace{-8pt}
\label{tab:design_space}
\begin{tabular}{lcccc}
\toprule
\multirow{2}{*}{Design Space} & \multicolumn{2}{c}{MNIST-4}    & \multicolumn{2}{c}{Fashion-2}  \\ 
                  & Yorktown      & Santiago      & Yorktown      & Santiago      \\ \midrule
`\zz+\rygate'              & \textbf{0.43} & 0.57          & 0.80          & \textbf{0.91} \\
\textbf{+\name}              & 0.34          & \textbf{0.60} & \textbf{0.83} & 0.86          \\ \midrule
`\rxyz'            & 0.57          & 0.61          & 0.88          & 0.89          \\
\textbf{+\name}              & \textbf{0.61} & \textbf{0.70} & \textbf{0.92} & \textbf{0.91} \\ \midrule
`\zx+\xx'             & 0.29          & 0.51          & \textbf{0.52} & 0.61          \\
\textbf{+\name}              & \textbf{0.38} & \textbf{0.64} & \textbf{0.52} & \textbf{0.89} \\ \midrule
`\rxyz+\uone+\cuthree'   & 0.28          & \textbf{0.25} & 0.48          & 0.50          \\
\textbf{+\name}              & \textbf{0.33} & 0.21          & \textbf{0.53} & \textbf{0.52} \\ \bottomrule
\end{tabular}%
\vspace{-8pt}
\end{table}

%% file: tables/tab_scalability.tex
\begin{table}
\centering
\renewcommand*{\arraystretch}{0.8}
\setlength{\tabcolsep}{13pt}
\captionof{table}{Scalable noise-aware training.}
\vspace{-8pt}
\begin{tabular}{lccc}
\toprule
Machine & Bogota & Santiago & Lima \\
\midrule
Noise-unaware & 0.74 & 0.97 & 0.87 \\
\textbf{\name} & \textbf{0.79} & \textbf{0.99} & \textbf{0.90} \\
\bottomrule
\end{tabular}%
\label{tab:scalability}
\vspace{-8pt}
\end{table}

%% file: tables/tab_ortho.tex
\begin{table}
\centering
\renewcommand*{\arraystretch}{0.8}
\setlength{\tabcolsep}{10pt}
\captionof{table}{Compatible with existing noise mitigation.}
\vspace{-8pt}
\label{tab:ortho}
\begin{tabular}{lcc}
\toprule
Method & MNIST-4 & Fashion-4 \\
\midrule
Normalization only & 0.78 & 0.81 \\
Normalization + Extrapolation & \textbf{0.81} & \textbf{0.83} \\ 
\bottomrule
\end{tabular}%
\vspace{-8pt}
\end{table}

%% file: tables/tab_readout_norm.tex
\begin{table}[t]
\centering
\renewcommand*{\arraystretch}{0.8}
\caption{Post-measurement norm. improves acc. \& SNR.}
\vspace{-10pt}
\label{tab:my-table}
\resizebox{\columnwidth}{!}{%
\begin{tabular}{cccccccccc}
\toprule
\multirow{3}{*}{\shortstack[c]{Quantum \\Devices\\$\downarrow$}} & \multicolumn{1}{l}{\multirow{3}{*}{\shortstack[c]{QNN\\ Models\\$\rightarrow$ }}} & \multicolumn{4}{c}{2 Blocks}                                                                                              & \multicolumn{4}{c}{4 Blocks}                                                \\ \cmidrule(lr){3-6}\cmidrule(lr){7-10}
                                  & \multicolumn{1}{l}{}                            & \multicolumn{2}{c}{$\times$2 Layers} & \multicolumn{2}{c}{$\times$8 Layers} &  \multicolumn{2}{l}{$\times$2 Layers} & \multicolumn{2}{l}{$\times$4 Layers} \\ \cmidrule(lr){3-4}\cmidrule(lr){5-6}\cmidrule(lr){7-8}\cmidrule(lr){9-10}
                                  & \multicolumn{1}{l}{}                            & Acc.              & SNR               & Acc.               & SNR               & Acc.              & SNR               & Acc.               & SNR              \\ \midrule
\multirow{2}{*}{Santiago}         & Baseline                                         & 0.61             & 6.15              & 0.52              & 1.79              & 0.57             & 6.96              & 0.62              & 4.20             \\  
                                  & \textbf{+Norm}                                            & \textbf{0.66}    & \textbf{15.69}    & \textbf{0.79}     & \textbf{4.85}      & \textbf{0.70}    & \textbf{11.36}    & \textbf{0.68}     & \textbf{6.55}    \\ \midrule
\multirow{2}{*}{Quito}            & Baseline                                         & 0.58             & 6.64              & 0.35              & 1.43                     & 0.60             & 3.98              & 0.29              & 1.73             \\  
                                  & \textbf{+Norm}                                            & \textbf{0.66}    & \textbf{13.92}    & \textbf{0.71}     & \textbf{2.98}    & \textbf{0.74}    & \textbf{12.26}    & \textbf{0.72}     & \textbf{4.54}    \\ \midrule
\multirow{2}{*}{Athens}           & Baseline                                         & 0.59             & 8.91              & 0.60              & 2.14             & 0.63                      & 9.52              & 0.55              & 3.54             \\  
                                  & \textbf{+Norm}                                            & \textbf{0.64}    & \textbf{20.27}    & \textbf{0.78}     & \textbf{3.47}     & \textbf{0.74}    & \textbf{14.07}    & \textbf{0.69}     & \textbf{6.09}    \\ \bottomrule
\end{tabular}%
\label{tab:norm}
}
\end{table}

%% file: figtex/fig_noise_model_ablation.tex
\begin{figure*}[t]
    \centering
    \includegraphics[width=\textwidth]{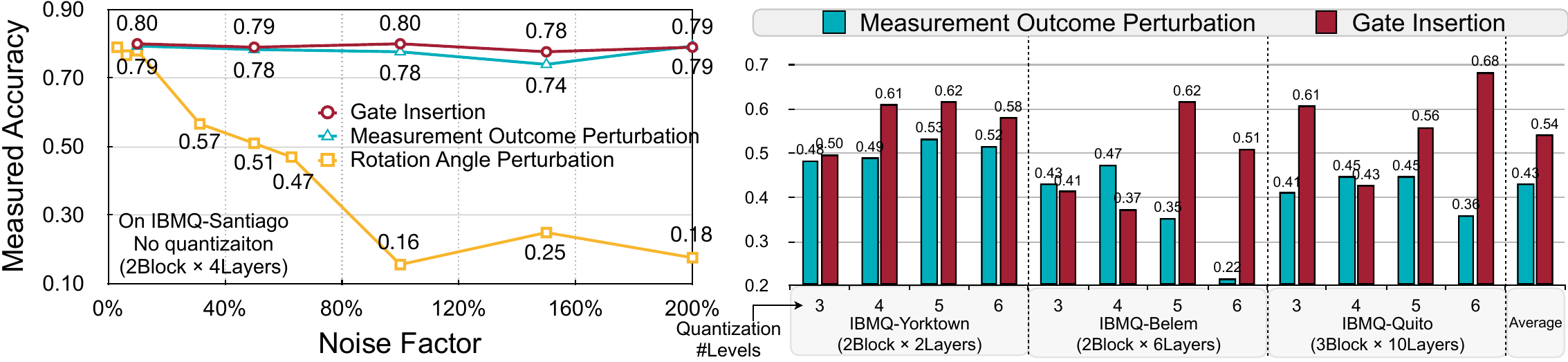}
    \vspace{-15pt}
    \caption{Ablation on different noise injection methods. \emph{Left:} Without quantization, gate insertion and measurement perturbation performs similar, both better than rotation angle perturbation. \emph{Right:} With quantization, gate insertion is better as perturbation effect can be canceled by quantization.}
    \label{fig:noisemodel}
    \vspace{-10pt}
\end{figure*}

%% file: figtex/fig_noiselevelcontour.tex
\begin{figure}[t]
    \centering
    \includegraphics[width=\columnwidth]{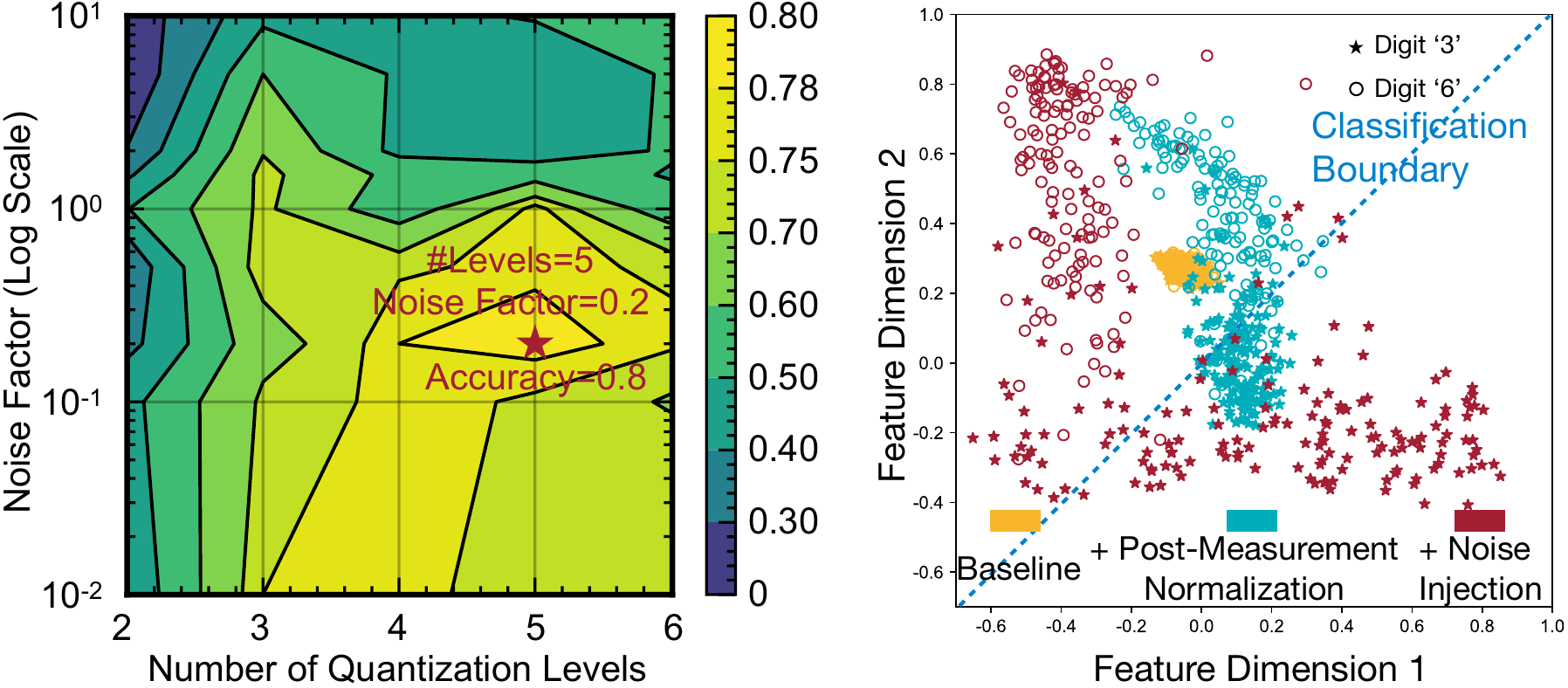}
    \vspace{-15pt}
    \caption{Left: Accuracy contours of quantization levels and noise factors. Right: Feature visualization.}
    \vspace{-10pt}
    \label{fig:noiselevelcontour}
\end{figure}

%% file: figtex/fig_breakdown.tex
\begin{figure}
\centering
\includegraphics[width=\columnwidth]{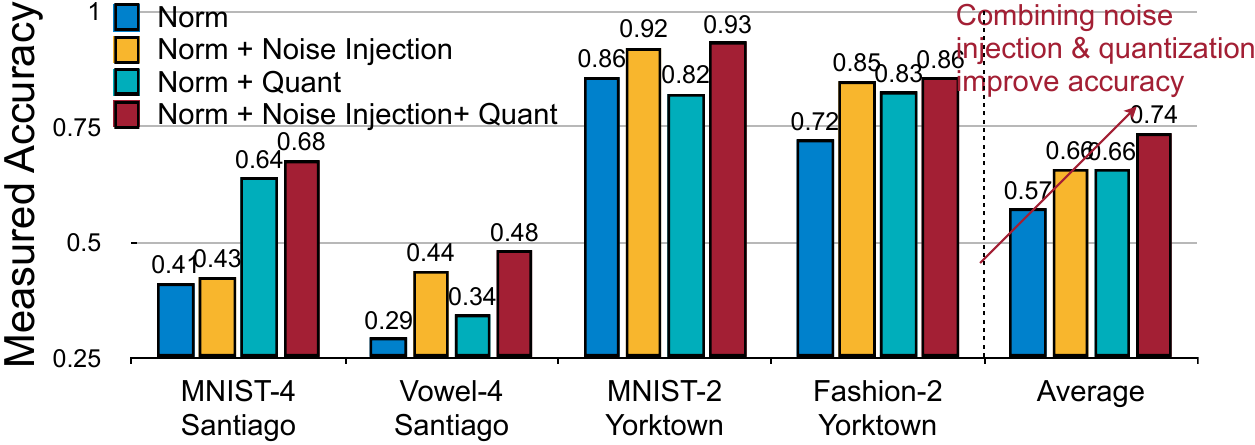}
\vspace{-15pt}
    \caption{Ablation of applying noise injection and quantization individually or jointly. }
    \label{fig:breakdown}
\vspace{-15pt}
\end{figure}

%% file: texts/5_conclusion.tex
\section{Conclusion}
\label{sec:conclusion}
\pqc is a promising candidate to demonstrate practical quantum advantages over classical approaches. The road to such advantage relies on: (1) the discovery of novel feature embedding that encodes classical data non-linearly, and (2) overcome the impact of quantum noise. This work focuses on the latter and show that a noise-aware training pipeline with post-measurement normalization, noise injection, and post-measurement quantization can elevate the \pqc robustness against arbitrary, realistic quantum noises.
We anticipate such robust \pqc being useful in exploring practical QC applications.

%% file: texts/7_acknowledgment.tex
\section*{Acknowledgment}

\small
We acknowledge NSF CAREER Award \#1943349, MIT-IBM Watson AI Lab, Qualcomm Innovation Fellowship. This work is funded in part by EPiQC, an NSF Expedition in Computing, under grants CCF-1730082/1730449; in part by STAQ under grant NSF Phy-1818914; in part by DOE grants DE-SC0020289 and DE-SC0020331; and in part by NSF OMA-2016136 and the Q-NEXT DOE NQI Center. We acknowledge the use of IBM Quantum services for this work.

%% file: texts/6_appendix.tex
\appendix

\section{Appendix}

\subsection{Quantum Basics and Quantum Noise}
\label{appen:basics}
A quantum circuit uses quantum bit (\emph{qubit}) to carry information, which is a linear combination of two basis state: $\ket{\psi}=\alpha\ket{0}+\beta\ket{1}$, for $\alpha,\beta\in\mathbb{C}$, satisfying $|\alpha|^2+|\beta|^2=1$.
An $n$-qubit system can represent a linear combination of 2$^n$ basis states.
A 2$^n$-length complex statevector of all combination coefficients is used to describe the circuit state.
In quantum computations, a sequence of quantum gates are applied to perform unitary transformation on the statevector, i.e., $\ket{\psi(\boldsymbol{x}, \boldsymbol{\theta})}=\cdots U_2(\boldsymbol{x},\theta_2)U_1(\boldsymbol{x},\theta_1)\ket{0}$, where $\boldsymbol{x}$ is the input data and $\boldsymbol{\theta}$ is the trainable parameters of rotation quantum gates. As such, the input data and trainable parameters are embedded in the quantum state $\ket{\psi(\boldsymbol{x}, \boldsymbol{\theta})}$.
Finally, the computation results are obtained by qubit readout/measurement which measures the probability of a qubit state $\ket{\psi}$ collapsing to either $\ket{0}$ (i.e., output $y=+1$) or $\ket{1}$ (i.e., output $y=-1$) according to $|\alpha|^2$ and $|\beta|^2$. With sufficient samples, we can compute the expectation value: $\E[y] = (+1)|\alpha|^2 + (-1)|\beta|^2$.
By cascading multiple blocks of quantum gates and measurements, a non-linear network can be constructed to perform ML tasks.

In real quantum computer systems, errors would likely occur due to imperfect control signals, unwanted interactions between qubits, or interference from the environment ~\citep{bruzewicz2019trapped, krantz2019quantum}. As a result, qubits undergo \emph{decoherence error} (spontaneous loss of its stored information) over time, and quantum gates introduce \emph{operation errors} (e.g., coherent errors and stochastic errors) into the system. These noisy systems need to be characterized~\citep{magesan2012characterizing} and calibrated~\citep{ibm_2021} frequently to mitigate the impact of noise on computation. Noise modeling helps to paint a realistic picture of the behavior and performance of a quantum computer and enables noisy simulations~\citep{ding2020quantum}. While exact modeling and simulation is challenging, many approximate strategies~\citep{magesan2012characterizing,wallman2016noise} have been developed based on Pauli/Clifford Twirling \citep{nielsen2002quantum,silva2008scalable}. 

\subsection{General Framework for Quantum Noise Analysis}
\label{appen:proof}
In this work, we examine how to characterize and mitigate the impact of quantum noises on quantum neural networks. We observe that the trainable quantum gates and post-measurement information processing play a huge role in boosting the algorithmic robustness to realistic quantum noises. In the following analysis, we restrict attention to: (1) a general (mixed) quantum state $\rho(\boldsymbol{x},\boldsymbol{\theta})$ resulting from a QNN for input data $\boldsymbol{x}$ and trainable parameters $\boldsymbol{\theta}$, (2) single-qubit measurement output, and (3) any fixed but unknown quantum noise. For multi-qubit quantum neural networks, similar analysis follows when considered qubit by qubit. 
\subsubsection{Measurement of Quantum Neural Networks}
\begin{definition} \textnormal{(Measurement procedure)}\textbf{.}
We measure a quantum state $\rho$ in the computational basis $\ket{b}: b \in \{0,1\}$ and output $z=+1$ if we obtain $\ket{0}\bra{0}$ and $z=-1$ if we obtain $\ket{1}\bra{1}$.
\end{definition}

The expectation value of such measurement contains useful information about the quantum state $\rho$:
\begin{align}
E_Z \equiv \E[z] = tr(Z\rho),
\end{align}
where $Z$ is the Pauli-Z matrix: $Z = (+1)\ket{0}\bra{0}+(-1)\ket{1}\bra{1}$ and $tr(\cdot)$ is the trace. We can estimate the expectation value by repeating the experiment by $s$ times, obtaining $z_1, \dotsc, z_s$, with each $z_j \in \{+1, -1\}$, and calculate their empirical mean: $y = \sum_{j=1}^s\frac{z_j}{s}$. Throughout this work, we use $s = 8192$ shots for the experiments to keep the variance low.
\begin{definition} \textnormal{(Noise processes)}\textbf{.}
A physical process (such as quantum noises) that can happen to a mixed quantum state $\rho$ can be described as a linear map: $\rho \rightarrow \mathcal{E}(\rho)$, such that
\begin{align}
\mathcal{E}(\rho) = \sum_{k}O_k\rho O_k^\dagger.
\end{align}
\end{definition}
The $O_k$'s are Kraus operators satisfying $\sum_k O_k^\dagger O_k = I$. The noise process for a quantum neural network can be challenging to characterize, as it depends not only on the input to the network but also on the qubits and quantum gates used in the network. We wish to analyze this noise process by decoupling its dependence on the input data. We assume that each $O_k$ has no explicit dependence on the classical input data; this is reasonable when we fix the model architecture. 
\subsubsection{Proof of Theorem \ref{th:linear_noise}}\label{proof:linear_noise}
Now we are ready to analyze the effect of quantum noises on the measurement result from a quantum neural network. For classical data $\boldsymbol{x_i}$ from the input data set $\boldsymbol{x}$, we construct a quantum neural network that embeds the classical data in the quantum state $\rho_i \equiv \rho(\boldsymbol{x_i},\boldsymbol{\theta_i})$ where $\boldsymbol{\theta_i}$ is some training parameters. The output of the network is the expectation value of the measurement outcome $E_{z,i}^* \equiv tr(Z\rho_i)$. However, in reality, the results are transformed by some unknown process $\rho_i \rightarrow \mathcal{E}(\rho_i)$. The goal is to quantify the impact of the quantum noise on the expectation value.
\renewcommand{\thetheorem}{\ref{th:linear_noise}}%
\begin{theorem}\textnormal{(formal version)}\textbf{.}
There exists some real parameters $\beta_i$ and $\gamma$, such that the expectation value of the measurement results $E_{z.i}$ for input data $\boldsymbol{x_i}$ with the presence of any valid quantum noise $\mathcal{E}(\rho)$ can be described as a linear map from the noiseless value $E_{z,i}^*$:
\begin{align}
E_{z,i} = \gamma E_{z, i}^* + \beta_i.
\end{align}
Here $\gamma$ is a scaling constant independent of the input data $\boldsymbol{x_i}$.
\end{theorem}
\begin{proof}
Suppose in the noiseless scenario the quantum state obtained from a quantum neural network is denoted as $\rho$, and the expectation value of its measurement results is $E_z^* = \E(\rho) = tr(\rho Z)$. Assume now the $\rho$ undergoes some quantum noise processes $\mathcal{E}(\rho) = \sum_k O_k \rho O_k^\dagger$. Therefore, in the presence of noise, the expectation value becomes:
\begin{align}
    E_z = \E[\mathcal{E}(\rho)] = tr(\mathcal{E}(\rho)Z) = \sum_{k}tr(O_k\rho O_k^\dagger Z)= \sum_{k}tr(\rho O_k^\dagger ZO_k),
\end{align}
where the third and fourth equality is from the properties of trace. We can further utilize the fact that an arbitrary quantum state can be expanded as:
\begin{align}
    \rho = \frac{1}{2}\left(tr(\rho)I+tr(X\rho)X+tr(Y\rho)Y+tr(Z\rho)Z\right).
\end{align}
If we denote $\Omega=\sum_{k}O_k^\dagger Z O_k$, we obtain
\begin{align}
    E_z = \frac{1}{2}tr(\Omega)+\frac{1}{2}tr(X\Omega)tr(X\rho)+\frac{1}{2}tr(Y\Omega)tr(Y\rho)+\frac{1}{2}tr(Z\Omega)tr(Z\rho).
\end{align}
Notice that $tr(\Omega) = 0$ and $tr(Z\rho) = E_z^*$. We can set $\gamma = \frac{1}{2}tr(Z\Omega) \in [-1,1]$ and $\beta_\rho = \frac{1}{2}tr(X\Omega)tr(X\rho) + \frac{1}{2}tr(Y\Omega)tr(Y\rho)$. We arrive at the linear map as desired.
\end{proof}
\renewcommand{\thetheorem}{\arabic{theorem}}

\subsection{Additional Experiments}

\subsubsection{Importance of hardware-specific noise model.}
\label{sec:appendix_hardwarespecific}
We train three \qnn models for Fashion-2 with the same architecture but different noise models from 3 devices and then deploy each model. Results in Table~\ref{tab:hardware_specific} show a diagonal pattern: the best accuracy is achieved when the noise model and inference device are the same. This is due to various noise magnitude and distribution on different devices. For instance, the gate error of Yorktown is 5\x larger than Santiago, so using Yorktown noise information for model running on Santiago is too large.
Therefore, a hardware-specific noise model is necessary for proper noise injection. However, this also marks the limitation of this work, as repeated training may be required when the noise model is updated. A future direction is to explore how to finetune already trained \qnn for fast adaption to a new noise setting, thus reducing the marginal cost.

\input{tables/tab_optthree}

\subsubsection{Compatibility with existing noise-adaptive compilation.}
\label{sec:appendix_noiseadaptive}
We further show the compatibility of \name with state-of-the-art noise-adaptive quantum compilation techniques. Specifically, we set the optimization level of Qiskit compiler to the highest 3, which enables noise-adaptive qubit mapping and instruction scheduling. Then we inference the \name trained model and compare the accuracy of MNIST-2 in Table~\ref{tab:optthree}. With noise-adaptive compilation, the accuracy of baseline models is improved. While on top of that, the \name can still provide over 10\% accuracy improvements, demonstrating the extensive applicability of our methods.

\subsubsection{Experiments on Fully Quantum Models}
\label{sec:appendix_fullyquantum}
\input{tables/tab_fully_quantum}
For the results in Section~\ref{sec:results}, the \qnn models contain multiple blocks. Here we further experiment on \emph{fully quantum} models which only contains one single block to show the strong generality of \name as in Table~\ref{tab:fully_quantum}. We select two fully quantum models, with three and six U3+CU3 layers, respectively, and experiment with six tasks on two machines. We apply the post-measurement normalization and quantization to the measurement outcomes of the last layer and use noise factor 0.5 and quantization level 6. No intermediate measurements are required. Our methods can still outperform baselines by \textbf{7.4\%} on average. Therefore, The noise injection can be applied to different kinds of variational quantum circuits, no matter whether the output of one layer is measured and passed to the next layer. Furthermore, the post-measurement normalization and quantization can also benefit various quantum circuits because they reduce the noise impact on measurement outcomes.

\subsubsection{Experiments on Effect of Number of Intermediate Measurements}
\label{sec:appendix_intermediate}
\input{tables/tab_number_blocks}

We also explore under the same number of parameters whether a fully quantum model is the best choice in the NISQ era. There exists a \emph{tradeoff} on the number of intermediate measurements as in Table~\ref{tab:n_blocks}. More measurements mean less noise impact because we can perform post-measurement normalization and quantization on measurement outcomes. However, measurements will collapse the state vector in the large Hilbert space back to the small classical space, hurting the model capacity. We perform experiments on the IBMQ-Santiago machine and find there exists a sweet spot to achieve the highest deployment accuracy: the best model contains 2 blocks and each has 3 layers.

Furthermore, we show direct accuracy comparisons between the original (with measurements in between) \qnn and fully-quantum \qnn in Table~\ref{tab:direct}. In each row, they have exactly the same dataset, same hardware. They have nearly the same architecture: same encoder/measurement, same gate sets, same layers, same number of parameters; the only difference is whether being measured and encoded back to quantum in the middle.

From the experimental results, we can see that under the total 6-layer setting, the 2 Block x 3 Layer can have better accuracy in most cases. This is because we perform normalization and quantization in the middle that can mitigate the noise impacts.

We also would like to emphasize that how to design the best architecture is not the main focus of our work. \name is architecture-agnostic and can be applied to various architectures to improve their robustness on real QC devices, as illustrated in paper Table~\ref{tab:design_space}.

\input{tables/tab_direct_comparison}

\subsubsection{Accuracy Gap between Using Noise Model and Real QC}
\label{sec:appendix_gap}
\input{tables/tab_gap}

To demonstrate the reliability of noise models, we show the accuracy gap of \qnn models evaluated with noise model and on real QC as in Table~\ref{tab:gap}. We can see that the accuracy gaps are typically smaller than 5\%, indicating \emph{high reliability} of noise models.

\subsubsection{Accuracy improvements comparison as number of classes increases}
\label{sec:appendix_improve}
Since we have different tasks with various number of classes, we compare the average accuracy improvements between them in Table~\ref{tab:improve}. We can see that the relative accuracy improvement on 10-class (230\%) is significantly higher than 4-class and 2-class. That of 4-class is also higher than 2-class. To improve the same absolute accuracy, it is clearly more difficult on a 10-class task than on a 2-class task. So \name is highly effective on 10-class tasks.

\input{tables/tab_improve}

\subsubsection{Experiments on Using Validation Set Statistics for Test Set}
\label{sec:appendix_validstats}
If the test batch size is small for the deployment on real QC hardware, then the statistics may not be accurate enough for post-measurement normalization. In this case, we can profile the statistics of the validation set on real hardware ahead of time and then use the validation set mean and std to normalize the test set measurement outcomes.

We experiment with three tasks, each on three quantum devices. We show the mean and std of measurement outcomes of each qubit on the validation set and test set as in Table~\ref{tab:usevalidstat}. We can see that the statistics of validation and test sets are similar. The last column of Table~\ref{tab:usevalidstat} shows the \emph{accuracy of test set} using statistics of the test set itself and validation set, respectively. In 9 benchmarks, the accuracy of two settings is very close. The average accuracy of using test set stats is 0.67; using validation set stats is 0.65.
\input{tables/tab_usevalidstat}

Therefore, using the statistics of validation set can bring similar accuracy to using statistics of test set itself; thus the \name can support small test batch size using validation set stats.

\subsection{Hyperparameters for main results}
\label{sec:appendix_hyper}

Table~\ref{tab:hyper} shows the detailed noise factor and quantization level for all the tasks in Table~\ref{tab:main}.
\input{tables/tab_hyper}

%% file: tables/tab_optthree.tex
\begin{table}[t]

\centering
\renewcommand*{\arraystretch}{1}
\setlength{\tabcolsep}{10pt}
\captionof{table}{Hardware-specific noise model can achieve best accuracy than other settings.}
\begin{tabular}{lcccc}
\toprule
\multicolumn{1}{r}{\small \underline{Use noise model of $\rightarrow$}} &  \multirow{2}{*}{Santiago}  & \multirow{2}{*}{Yorktown} & \multirow{2}{*}{Lima}  \\
        {\small \underline{Inference on $\downarrow$}} & & &  \\
        \midrule
        
        Santiago & \cellcolor{blue!20} \textbf{0.90} & \cellcolor{myred!20}0.55  & 
        \cellcolor{blue!20} \textbf{0.91} \\
        Yorktown & \cellcolor{myred!20} 0.41 &  \cellcolor{blue!20} \textbf{0.55} & \cellcolor{myred!20}0.5 \\
        
        Lima &\cellcolor{myred!20} 0.76 & \cellcolor{myred!20}0.76 & \cellcolor{blue!20}\textbf{0.89} \\                        \bottomrule
        
\end{tabular}%
\label{tab:hardware_specific}
\end{table}

\begin{table}
\centering
\renewcommand*{\arraystretch}{1}
\setlength{\tabcolsep}{6pt}
\captionof{table}{MNIST-2 accuracy with noise-adaptive compilation enabled (Qiskit optimization level=3).}
\begin{tabular}{lcccc}
\toprule
Method & Santiago  & Yorktown & Belem & Athens \\

        \midrule
        Baseline & 0.68  & 0.83 & 0.83 & 0.54 \\
        $+$Norm & 0.87 & 0.86 & 0.91 & 0.51 \\
        
        $+$Noise \& Quant & \textbf{0.92} & \textbf{0.92} & \textbf{0.91} & \textbf{0.93} \\                        \bottomrule
\end{tabular}%
\label{tab:optthree}


\end{table}

%% file: tables/tab_fully_quantum.tex
\begin{table*}[t]
\centering
\renewcommand*{\arraystretch}{1}
\setlength{\tabcolsep}{10pt}

\caption{Effect of \name on fully quantum models.}
\label{tab:fully_quantum}
\begin{tabular}{lcccccccc}
\toprule
IBMQ Machine & Model & Method & MNIST-4 & Fashion-4 & Vowel-4 & MNIST-2 & Fashion-2 & Cifar-2 \\
\midrule
\multirow{2}{*}{Santiago} & \multirow{2}{*}{3 Layer} & Baseline & 0.64 & 0.78 & 0.41 & 0.94 & 0.89 & \textbf{0.59} \\
& & \textbf{\name} & \textbf{0.78} & \textbf{0.82} & \textbf{0.53} & \textbf{0.96} & \textbf{0.90} & 0.58 \\
\midrule
\multirow{2}{*}{Santiago} & \multirow{2}{*}{6 Layer} & Baseline & 0.61 & 0.37 & 0.22 & 0.51 & 0.52 & 0.52 \\
& & \textbf{\name} &  \textbf{0.62} &  \textbf{0.69} &  \textbf{0.22} &  \textbf{0.84} &  \textbf{0.89} &  \textbf{0.56} \\
\midrule
\multirow{2}{*}{Yorktown} & \multirow{2}{*}{3 Layer} & Baseline & 0.49 & 0.53 & 0.4 & 0.88 & 0.85 & 0.51 \\
& & \textbf{\name} & \textbf{0.55} & \textbf{0.66} & \textbf{0.42} & \textbf{0.9} & \textbf{0.91} & \textbf{0.55} \\
\midrule
\multirow{2}{*}{Yorktown} & \multirow{2}{*}{6 Layer} & Baseline & 0.22 & 0.33 &  \textbf{0.26} & 0.73 & 0.80 &  \textbf{0.54} \\
& & \textbf{\name} & \textbf{0.42} & \textbf{0.35} & 0.25 & \textbf{0.78} & \textbf{0.80} & 0.52 \\
\midrule
\multirow{2}{*}{Belem} & \multirow{2}{*}{3 Layer} & Baseline & 0.53 & \textbf{0.60} & 0.37 & 0.64 & 0.81 & 0.51
 \\
& & \textbf{\name} & \textbf{0.58} & 0.42 & \textbf{0.39} & \textbf{0.93} & \textbf{0.85} & \textbf{0.55} \\
\midrule
\multirow{2}{*}{Belem} & \multirow{2}{*}{6 Layer} & Baseline & 0.27 & 0.18 & 0.21 & 0.54 & 0.48 & 0.43
 \\
& & \textbf{\name} & \textbf{0.43} & \textbf{0.31} & \textbf{0.22} & \textbf{0.54} & \textbf{0.54} & \textbf{0.52} \\

\bottomrule
\end{tabular}%
\end{table*}

%% file: tables/tab_number_blocks.tex
\begin{table}[t]
\centering

\caption{Effect of number of intermediate measurements.}
\label{tab:n_blocks}
\begin{tabular}{lcccc}
\toprule
\multirow{2}{*}{Task} & 1 Block & 2 Blocks  & 3 Blocks  & 6 Blocks\\
& \x 6 Layers & \x 3 Layers & \x 2 Layers & \x 1 Layer \\
\midrule
MNIST-4 & 0.62 & \textbf{0.74} & 0.71 & 0.66 \\
Fashion-4 & 0.69 & \textbf{0.82} & 0.78 & 0.68 \\
\bottomrule
\end{tabular}%
\end{table}

%% file: tables/tab_direct_comparison.tex
\begin{table}[t]
\centering
\renewcommand*{\arraystretch}{1}
\setlength{\tabcolsep}{5pt}
\caption{Direct comparison between \qnn models with measurement in between and fully-quantum \qnn models.}
\label{tab:direct}
\begin{tabular}{llcc}
\toprule
Machine & Task & Fully-Quantum (6L) & Original (2B \x 3L) \\
\midrule
Santiago & MNIST-4 & 0.62 & \textbf{0.74} \\
Santiago & Fashion-4 & 0.69 & \textbf{0.82} \\ 
Santiago & MNIST-2 & 0.84 & \textbf{0.86} \\
\midrule
Belem & MNIST-4 & \textbf{0.43} & 0.37 \\
Belem & Fashion-4 & 0.31 & \textbf{0.34} \\ 
Belem & MNIST-2 & 0.54 & \textbf{0.60} \\ 
\bottomrule
\end{tabular}%
\end{table}

%% file: tables/tab_gap.tex
\begin{table*}[t]
\centering

\caption{Accuracy gap between evaluation using noise model and real QC.}
\label{tab:gap}
\begin{tabular}{lllcccccc}
\toprule
Machine & Model & Method & MNIST-4 & Fashion-4 & Vowel-4 & MNIST-2 & Fashion-2 & Cifar-2 \\
\midrule
\multirow{2}{*}{Santiago} & \multirow{2}{*}{\shortstack[l]{2 Blocks \\ \x 12 Layer}} & Noise model & 0.73 & 0.74 & 0.51 & 0.95 & 0.92 & 0.65 \\
& & Real QC & 0.68 & 0.75 & 0.48 & 0.94 & 0.88 & 0.59 \\
\midrule
\multirow{2}{*}{Yorktown} & \multirow{2}{*}{\shortstack[l]{2 Blocks \\ \x 2 Layer}} & Noise model & 0.68 & 0.7 & 0.44 & 0.92 & 0.90 & 0.59 \\
& & Real QC & 0.62 & 0.65 & 0.44 & 0.93 & 0.86 & 0.60 \\
\midrule
\multirow{2}{*}{Belem} & \multirow{2}{*}{\shortstack[l]{2 Blocks \\ \x 6 Layer}} & Noise model & 0.64 & 0.72 & 0.41 & 0.96 & 0.82 & 0.64 \\
& & Real QC & 0.58 & 0.62 & 0.41 & 0.88 & 0.8 & 0.61 \\

\bottomrule
\end{tabular}%
\end{table*}

%% file: tables/tab_improve.tex
\begin{table*}[t]
\centering
\renewcommand*{\arraystretch}{1}
\setlength{\tabcolsep}{10pt}

\caption{Improvements are still significant as the number of classes increases.}
\label{tab:improve}
\begin{tabular}{lcccc}
\toprule
Task Average Accuracy & Baseline & \name & Absolute Improvement & Relative Improvement\\
\midrule
2-classification & 0.58 & 0.76 & 0.28 & 48\% \\
4-classification & 0.31 & 0.57 & 0.26 & 84\% \\
10-classification & 0.1 & 0.33 & 0.23 & 230\% \\
\bottomrule
\end{tabular}%
\end{table*}

%% file: tables/tab_usevalidstat.tex
\begin{table*}[t]
\centering
\renewcommand*{\arraystretch}{1}
\setlength{\tabcolsep}{6pt}
\caption{Statistics of test and validation set; Accuracy of test set using test stats and validation stats.}
\label{tab:usevalidstat}

\begin{tabular}{lcccc}
\toprule
\multirow{1}{*}{Task} & Stats & \multicolumn{1}{c}{MEAN}    & \multicolumn{1}{c}{STD} & \multicolumn{1}{c}{Accuracy}  \\ 
\midrule
\multirow{2}{*}{Fashion-4-Santiago}   & Test Stats & [ 0.0469,  0.0025, -0.0581, -0.0191]  & [0.0868, 0.0496, 0.1021, 0.1152]  & 0.75 \\
&  Valid Stats   & [ 0.0679,  0.0025, -0.0519, -0.0473]          & [0.0915, 0.0448, 0.0884, 0.1114] & 0.70 \\ 
\midrule
\multirow{2}{*}{Fashion-4-Yorktown}   & Test Stats &  [-0.0396,  0.0478,  0.0995,  0.1375] & [0.1279, 0.3368, 0.1761, 0.1538]  & 0.65 \\
&  Valid Stats   & [-0.0362,  0.0771,  0.0965,  0.1535] & [0.1230, 0.3233, 0.1835, 0.1584] & 0.65 \\ 
\midrule
\multirow{2}{*}{Fashion-4-Belem} & Test Stats & [ 0.1118,  0.0075,  0.0901, -0.0005] & [0.0868, 0.1511, 0.1391, 0.2039] & 0.62 \\
&  Valid Stats   & [ 0.1508, -0.0130,  0.0533,  0.0478] & [0.0882, 0.1298, 0.1315, 0.1401] & 0.53 \\ 
\midrule
\multirow{2}{*}{Vowel-4-Santiago}   & Test Stats & [0.1091, 0.0526, 0.0290, 0.2172] & [0.0551, 0.0260, 0.0554, 0.0422] & 0.48  \\
&  Valid Stats & [0.1042, 0.0698, 0.0458, 0.1951] & [0.0418, 0.0226, 0.0443, 0.0362] & 0.43 \\ 
\midrule
\multirow{2}{*}{Vowel-4-Yorktown} & Test Stats & [ 0.0900, -0.3700, -0.2524,  0.1645] & [0.0997, 0.0580, 0.0663, 0.1198] & 0.44  \\
&  Valid Stats   & [ 0.0841, -0.3869, -0.2948,  0.1736] & [0.0946, 0.0651, 0.0615, 0.1199] & 0.41 \\ 
\midrule
\multirow{2}{*}{Vowel-4-Belem}   & Test Stats & [0.0115, 0.0800, 0.1703, 0.1775] & [0.0171, 0.0411, 0.0518, 0.0293] & 0.41 \\
&  Valid Stats   & [-0.0213,  0.0459,  0.1930,  0.1628] & [0.0145, 0.0335, 0.0478, 0.0263] & 0.40 \\ 
\midrule
\multirow{2}{*}{MNIST-2-Santiago}   & Test Stats & [-0.0581, -0.0657,  0.0088,  0.0170] & [0.0737, 0.1090, 0.1561, 0.1351] & 0.94 \\
&  Valid Stats   & [-0.0739,  0.0001, -0.0113,  0.00239] & [0.0666, 0.0840, 0.1468, 0.1167]  & 0.95 \\ 
\midrule
\multirow{2}{*}{MNIST-2-Yorktown}   & Test Stats & [ 0.0892, -0.0007,  0.0548,  0.0485] & [0.1281, 0.3501, 0.2100, 0.2975] & 0.93 \\
&  Valid Stats   & [0.0704, 0.0536, 0.0204, 0.1043] & [0.1377, 0.3813, 0.2596, 0.2955] & 0.91 \\ 
\midrule
\multirow{2}{*}{MNIST-2-Belem}   & Test Stats & [-0.0649,  0.1949,  0.0540,  0.1313] & [0.0856, 0.1137, 0.1553, 0.1688] & 0.88 \\
&  Valid Stats   & [-0.0540,  0.2074,  0.0744,  0.1872] & [0.0561, 0.1008, 0.1345, 0.1103] & 0.91  \\ 
\midrule
\multirow{2}{*}{\textbf{Average}} & Test Stats  & --- & --- & \textbf{0.67}  \\

& Valid Stats & --- & --- & \textbf{0.65} \\

 \bottomrule
\end{tabular}%

\end{table*}

%% file: tables/tab_hyper.tex
\begin{table*}[t]
\centering

\caption{Hyperparameters of Table~\ref{tab:main}.}
\label{tab:hyper}

\begin{tabular}{lcccccc}
\toprule
Task, (noise-factor, quantization level) & MNIST-4 & Fashion-4 & Vowel-4 & MNIST-2 & Fashion-2 & Cifar-2 \\
\midrule
QNN (2 Blocks \x 12 Layers) on Santiago & (1, 3) & (0.5, 6) & (0.5, 6) & (1, 4) & (1, 6) & (0.5, 6) \\
QNN (2 Blocks x 2 Layers) on Yorktown& (0.5, 6)& (1, 5)& (0.1, 5)& (0.5, 5)& (0.1, 6)& (0.1, 3) \\
QNN (2 Blocks x 6 Layers) on Belem& (0.5, 5)& (1.5, 6)& (0.5, 3)& (0.5, 5)& (0.1, 4)& (0.5, 6) \\
QNN (3 Blocks x 10 Layers) on Athens & (0.1, 6)& (0.1, 5)& (0.5, 6)& (0.1, 6)& (0.5, 6)& (0.1, 6) \\
\midrule
\midrule
Task, (noise-factor, quantization level) & MNIST-10 & Fashion-10 &&&& \\ 
\midrule
QNN (2 Blocks x 2 Layers) on Melbourne & (0.1, 6) & (0.1, 5) &&&& \\

\bottomrule
\end{tabular}%

\end{table*}

%% file: main.bbl

\begin{thebibliography}{30}


\ifx \showCODEN    \undefined \def \showCODEN     #1{\unskip}     \fi
\ifx \showDOI      \undefined \def \showDOI       #1{#1}\fi
\ifx \showISBNx    \undefined \def \showISBNx     #1{\unskip}     \fi
\ifx \showISBNxiii \undefined \def \showISBNxiii  #1{\unskip}     \fi
\ifx \showISSN     \undefined \def \showISSN      #1{\unskip}     \fi
\ifx \showLCCN     \undefined \def \showLCCN      #1{\unskip}     \fi
\ifx \shownote     \undefined \def \shownote      #1{#1}          \fi
\ifx \showarticletitle \undefined \def \showarticletitle #1{#1}   \fi
\ifx \showURL      \undefined \def \showURL       {\relax}        \fi
\providecommand\bibfield[2]{#2}
\providecommand\bibinfo[2]{#2}
\providecommand\natexlab[1]{#1}
\providecommand\showeprint[2][]{arXiv:#2}

\bibitem[\protect\citeauthoryear{Amin, Andriyash, Rolfe, Kulchytskyy, and
  Melko}{Amin et~al\mbox{.}}{2018}]%
        {amin2018quantum}
\bibfield{author}{\bibinfo{person}{Mohammad~H Amin}, \bibinfo{person}{Evgeny
  Andriyash}, \bibinfo{person}{Jason Rolfe}, \bibinfo{person}{Bohdan
  Kulchytskyy}, {and} \bibinfo{person}{Roger Melko}.}
  \bibinfo{year}{2018}\natexlab{}.
\newblock \showarticletitle{Quantum boltzmann machine}.
\newblock \bibinfo{journal}{\emph{Physical Review X}} \bibinfo{volume}{8},
  \bibinfo{number}{2} (\bibinfo{year}{2018}).
\newblock


\bibitem[\protect\citeauthoryear{Bruzewicz, Chiaverini, McConnell, and
  Sage}{Bruzewicz et~al\mbox{.}}{2019}]%
        {bruzewicz2019trapped}
\bibfield{author}{\bibinfo{person}{Colin~D Bruzewicz}, \bibinfo{person}{John
  Chiaverini}, \bibinfo{person}{Robert McConnell}, {and}
  \bibinfo{person}{Jeremy~M Sage}.} \bibinfo{year}{2019}\natexlab{}.
\newblock \showarticletitle{Trapped-ion quantum computing: Progress and
  challenges}.
\newblock \bibinfo{journal}{\emph{Applied Physics Reviews}}
  \bibinfo{volume}{6}, \bibinfo{number}{2} (\bibinfo{year}{2019}),
  \bibinfo{pages}{021314}.
\newblock


\bibitem[\protect\citeauthoryear{Crooks}{Crooks}{2019}]%
        {crooks2019gradients}
\bibfield{author}{\bibinfo{person}{Gavin~E Crooks}.}
  \bibinfo{year}{2019}\natexlab{}.
\newblock \showarticletitle{Gradients of parameterized quantum gates using the
  parameter-shift rule and gate decomposition}.
\newblock \bibinfo{journal}{\emph{arXiv preprint 1905.13311}}
  (\bibinfo{year}{2019}).
\newblock


\bibitem[\protect\citeauthoryear{Ding and Chong}{Ding and Chong}{2020}]%
        {ding2020quantum}
\bibfield{author}{\bibinfo{person}{Yongshan Ding} {and}
  \bibinfo{person}{Frederic~T Chong}.} \bibinfo{year}{2020}\natexlab{}.
\newblock \showarticletitle{Quantum computer systems: Research for noisy
  intermediate-scale quantum computers}.
\newblock \bibinfo{journal}{\emph{Synthesis Lectures on Computer Architecture}}
  \bibinfo{volume}{15}, \bibinfo{number}{2} (\bibinfo{year}{2020}),
  \bibinfo{pages}{1--227}.
\newblock


\bibitem[\protect\citeauthoryear{Farhi and Neven}{Farhi and Neven}{2018}]%
        {farhi2018classification}
\bibfield{author}{\bibinfo{person}{Edward Farhi} {and} \bibinfo{person}{Hartmut
  Neven}.} \bibinfo{year}{2018}\natexlab{}.
\newblock \showarticletitle{Classification with quantum neural networks on near
  term processors}.
\newblock \bibinfo{journal}{\emph{arXiv preprint arXiv:1802.06002}}
  (\bibinfo{year}{2018}).
\newblock


\bibitem[\protect\citeauthoryear{Han, Mao, and Dally}{Han
  et~al\mbox{.}}{2015}]%
        {han2015deep}
\bibfield{author}{\bibinfo{person}{Song Han}, \bibinfo{person}{Huizi Mao},
  {and} \bibinfo{person}{William~J Dally}.} \bibinfo{year}{2015}\natexlab{}.
\newblock \showarticletitle{Deep compression: Compressing deep neural networks
  with pruning, trained quantization and huffman coding}.
\newblock \bibinfo{journal}{\emph{arXiv preprint arXiv:1510.00149}}
  (\bibinfo{year}{2015}).
\newblock


\bibitem[\protect\citeauthoryear{Henderson, Shakya, Pradhan, and
  Cook}{Henderson et~al\mbox{.}}{2020}]%
        {henderson2020quanvolutional}
\bibfield{author}{\bibinfo{person}{Maxwell Henderson},
  \bibinfo{person}{Samriddhi Shakya}, \bibinfo{person}{Shashindra Pradhan},
  {and} \bibinfo{person}{Tristan Cook}.} \bibinfo{year}{2020}\natexlab{}.
\newblock \showarticletitle{Quanvolutional neural networks: powering image
  recognition with quantum circuits}.
\newblock \bibinfo{journal}{\emph{Quantum Machine Intelligence}}
  \bibinfo{volume}{2}, \bibinfo{number}{1} (\bibinfo{year}{2020}),
  \bibinfo{pages}{1--9}.
\newblock


\bibitem[\protect\citeauthoryear{IBM}{IBM}{[n.\,d.]}]%
        {ibmq}
\bibfield{author}{\bibinfo{person}{IBM}.} \bibinfo{year}{[n.\,d.]}\natexlab{}.
\newblock \showarticletitle{IBM Quantum}.
\newblock  (\bibinfo{year}{[n.\,d.]}).
\newblock
\urldef\tempurl%
\url{https://quantum-computing.ibm.com/}
\showURL{%
\tempurl}


\bibitem[\protect\citeauthoryear{IBM}{IBM}{2021}]%
        {ibm_2021}
\bibfield{author}{\bibinfo{person}{Qiskit IBM}.}
  \bibinfo{year}{2021}\natexlab{}.
\newblock
\newblock
\urldef\tempurl%
\url{https://qiskit.org/textbook/ch-quantum-hardware/calibrating-qubits-pulse.html}
\showURL{%
\tempurl}


\bibitem[\protect\citeauthoryear{Ioffe and Szegedy}{Ioffe and Szegedy}{2015}]%
        {ioffe2015batch}
\bibfield{author}{\bibinfo{person}{Sergey Ioffe} {and}
  \bibinfo{person}{Christian Szegedy}.} \bibinfo{year}{2015}\natexlab{}.
\newblock \showarticletitle{Batch normalization: Accelerating deep network
  training by reducing internal covariate shift}. In
  \bibinfo{booktitle}{\emph{ICML}}. PMLR, \bibinfo{pages}{448--456}.
\newblock


\bibitem[\protect\citeauthoryear{Jiang, Xiong, and Shi}{Jiang
  et~al\mbox{.}}{2021}]%
        {jiang2021co}
\bibfield{author}{\bibinfo{person}{Weiwen Jiang}, \bibinfo{person}{Jinjun
  Xiong}, {and} \bibinfo{person}{Yiyu Shi}.} \bibinfo{year}{2021}\natexlab{}.
\newblock \showarticletitle{A co-design framework of neural networks and
  quantum circuits towards quantum advantage}.
\newblock \bibinfo{journal}{\emph{Nature communications}} \bibinfo{volume}{12},
  \bibinfo{number}{1} (\bibinfo{year}{2021}), \bibinfo{pages}{1--13}.
\newblock


\bibitem[\protect\citeauthoryear{Krantz, Kjaergaard, Yan, Orlando, Gustavsson,
  and Oliver}{Krantz et~al\mbox{.}}{2019}]%
        {krantz2019quantum}
\bibfield{author}{\bibinfo{person}{Philip Krantz}, \bibinfo{person}{Morten
  Kjaergaard}, \bibinfo{person}{Fei Yan}, \bibinfo{person}{Terry~P Orlando},
  \bibinfo{person}{Simon Gustavsson}, {and} \bibinfo{person}{William~D
  Oliver}.} \bibinfo{year}{2019}\natexlab{}.
\newblock \showarticletitle{A quantum engineer's guide to superconducting
  qubits}.
\newblock \bibinfo{journal}{\emph{Applied Physics Reviews}}
  \bibinfo{volume}{6}, \bibinfo{number}{2} (\bibinfo{year}{2019}),
  \bibinfo{pages}{021318}.
\newblock


\bibitem[\protect\citeauthoryear{Krizhevsky, Nair, and Hinton}{Krizhevsky
  et~al\mbox{.}}{[n.\,d.]}]%
        {cifar10}
\bibfield{author}{\bibinfo{person}{Alex Krizhevsky}, \bibinfo{person}{Vinod
  Nair}, {and} \bibinfo{person}{Geoffrey Hinton}.}
  \bibinfo{year}{[n.\,d.]}\natexlab{}.
\newblock \showarticletitle{CIFAR-10 (Canadian Institute for Advanced
  Research)}.
\newblock  (\bibinfo{year}{[n.\,d.]}).
\newblock


\bibitem[\protect\citeauthoryear{{Lecun}, {Bottou}, {Bengio}, and
  {Haffner}}{{Lecun} et~al\mbox{.}}{1998}]%
        {726791}
\bibfield{author}{\bibinfo{person}{Y. {Lecun}}, \bibinfo{person}{L. {Bottou}},
  \bibinfo{person}{Y. {Bengio}}, {and} \bibinfo{person}{P. {Haffner}}.}
  \bibinfo{year}{1998}\natexlab{}.
\newblock \showarticletitle{Gradient-based learning applied to document
  recognition}.
\newblock \bibinfo{journal}{\emph{Proc. IEEE}} \bibinfo{volume}{86},
  \bibinfo{number}{11} (\bibinfo{year}{1998}), \bibinfo{pages}{2278--2324}.
\newblock


\bibitem[\protect\citeauthoryear{Liang, Wang, Yang, Yang, Shi, and Jiang}{Liang
  et~al\mbox{.}}{2021}]%
        {liang2021can}
\bibfield{author}{\bibinfo{person}{Zhiding Liang}, \bibinfo{person}{Zhepeng
  Wang}, \bibinfo{person}{Junhuan Yang}, \bibinfo{person}{Lei Yang},
  \bibinfo{person}{Yiyu Shi}, {and} \bibinfo{person}{Weiwen Jiang}.}
  \bibinfo{year}{2021}\natexlab{}.
\newblock \showarticletitle{Can Noise on Qubits Be Learned in Quantum Neural
  Network? A Case Study on QuantumFlow}. In \bibinfo{booktitle}{\emph{ICCAD}}.
  IEEE, \bibinfo{pages}{1--7}.
\newblock


\bibitem[\protect\citeauthoryear{Lin, Gan, and Han}{Lin et~al\mbox{.}}{2019}]%
        {lin2019defensive}
\bibfield{author}{\bibinfo{person}{Ji Lin}, \bibinfo{person}{Chuang Gan}, {and}
  \bibinfo{person}{Song Han}.} \bibinfo{year}{2019}\natexlab{}.
\newblock \showarticletitle{Defensive quantization: When efficiency meets
  robustness}.
\newblock \bibinfo{journal}{\emph{arXiv preprint arXiv:1904.08444}}
  (\bibinfo{year}{2019}).
\newblock


\bibitem[\protect\citeauthoryear{Lloyd, Schuld, Ijaz, Izaac, and
  Killoran}{Lloyd et~al\mbox{.}}{2020}]%
        {lloyd2020quantum}
\bibfield{author}{\bibinfo{person}{Seth Lloyd}, \bibinfo{person}{Maria Schuld},
  \bibinfo{person}{Aroosa Ijaz}, \bibinfo{person}{Josh Izaac}, {and}
  \bibinfo{person}{Nathan Killoran}.} \bibinfo{year}{2020}\natexlab{}.
\newblock \showarticletitle{Quantum embeddings for machine learning}.
\newblock \bibinfo{journal}{\emph{arXiv preprint 2001.03622}}
  (\bibinfo{year}{2020}).
\newblock


\bibitem[\protect\citeauthoryear{Magesan, Gambetta, and Emerson}{Magesan
  et~al\mbox{.}}{2012}]%
        {magesan2012characterizing}
\bibfield{author}{\bibinfo{person}{Easwar Magesan}, \bibinfo{person}{Jay~M
  Gambetta}, {and} \bibinfo{person}{Joseph Emerson}.}
  \bibinfo{year}{2012}\natexlab{}.
\newblock \showarticletitle{Characterizing quantum gates via randomized
  benchmarking}.
\newblock \bibinfo{journal}{\emph{Physical Review A}} \bibinfo{volume}{85},
  \bibinfo{number}{4} (\bibinfo{year}{2012}), \bibinfo{pages}{042311}.
\newblock


\bibitem[\protect\citeauthoryear{Matsuoka}{Matsuoka}{1992}]%
        {155944}
\bibfield{author}{\bibinfo{person}{K. Matsuoka}.}
  \bibinfo{year}{1992}\natexlab{}.
\newblock \showarticletitle{Noise injection into inputs in back-propagation
  learning}.
\newblock \bibinfo{journal}{\emph{IEEE Transactions on Systems, Man, and
  Cybernetics}} \bibinfo{volume}{22}, \bibinfo{number}{3}
  (\bibinfo{year}{1992}), \bibinfo{pages}{436--440}.
\newblock


\bibitem[\protect\citeauthoryear{McClean, Boixo, Smelyanskiy, Babbush, and
  Neven}{McClean et~al\mbox{.}}{2018}]%
        {mcclean2018barren}
\bibfield{author}{\bibinfo{person}{Jarrod~R McClean}, \bibinfo{person}{Sergio
  Boixo}, \bibinfo{person}{Vadim~N Smelyanskiy}, \bibinfo{person}{Ryan
  Babbush}, {and} \bibinfo{person}{Hartmut Neven}.}
  \bibinfo{year}{2018}\natexlab{}.
\newblock \showarticletitle{Barren plateaus in quantum neural network training
  landscapes}.
\newblock \bibinfo{journal}{\emph{Nature communications}} \bibinfo{volume}{9},
  \bibinfo{number}{1} (\bibinfo{year}{2018}), \bibinfo{pages}{1--6}.
\newblock


\bibitem[\protect\citeauthoryear{Nielsen and Chuang}{Nielsen and
  Chuang}{2002}]%
        {nielsen2002quantum}
\bibfield{author}{\bibinfo{person}{Michael~A Nielsen} {and}
  \bibinfo{person}{Isaac Chuang}.} \bibinfo{year}{2002}\natexlab{}.
\newblock \bibinfo{title}{Quantum computation and quantum information}.
\newblock
\newblock


\bibitem[\protect\citeauthoryear{Silva, Magesan, Kribs, and Emerson}{Silva
  et~al\mbox{.}}{2008}]%
        {silva2008scalable}
\bibfield{author}{\bibinfo{person}{Marcus Silva}, \bibinfo{person}{Easwar
  Magesan}, \bibinfo{person}{David~W Kribs}, {and} \bibinfo{person}{Joseph
  Emerson}.} \bibinfo{year}{2008}\natexlab{}.
\newblock \showarticletitle{Scalable protocol for identification of correctable
  codes}.
\newblock \bibinfo{journal}{\emph{Physical Review A}} \bibinfo{volume}{78},
  \bibinfo{number}{1} (\bibinfo{year}{2008}), \bibinfo{pages}{012347}.
\newblock


\bibitem[\protect\citeauthoryear{Temme, Bravyi, and Gambetta}{Temme
  et~al\mbox{.}}{2017}]%
        {temme2017error}
\bibfield{author}{\bibinfo{person}{Kristan Temme}, \bibinfo{person}{Sergey
  Bravyi}, {and} \bibinfo{person}{Jay~M Gambetta}.}
  \bibinfo{year}{2017}\natexlab{}.
\newblock \showarticletitle{Error mitigation for short-depth quantum circuits}.
\newblock \bibinfo{journal}{\emph{Physical review letters}}
  \bibinfo{volume}{119}, \bibinfo{number}{18} (\bibinfo{year}{2017}),
  \bibinfo{pages}{180509}.
\newblock


\bibitem[\protect\citeauthoryear{Wallman and Emerson}{Wallman and
  Emerson}{2016}]%
        {wallman2016noise}
\bibfield{author}{\bibinfo{person}{Joel~J Wallman} {and}
  \bibinfo{person}{Joseph Emerson}.} \bibinfo{year}{2016}\natexlab{}.
\newblock \showarticletitle{Noise tailoring for scalable quantum computation
  via randomized compiling}.
\newblock \bibinfo{journal}{\emph{Physical Review A}} \bibinfo{volume}{94},
  \bibinfo{number}{5} (\bibinfo{year}{2016}), \bibinfo{pages}{052325}.
\newblock


\bibitem[\protect\citeauthoryear{Wang, Ding, Gu, Lin, Pan, Chong, and Han}{Wang
  et~al\mbox{.}}{2022a}]%
        {wang2021quantumnas}
\bibfield{author}{\bibinfo{person}{Hanrui Wang}, \bibinfo{person}{Yongshan
  Ding}, \bibinfo{person}{Jiaqi Gu}, \bibinfo{person}{Yujun Lin},
  \bibinfo{person}{David~Z Pan}, \bibinfo{person}{Frederic~T Chong}, {and}
  \bibinfo{person}{Song Han}.} \bibinfo{year}{2022}\natexlab{a}.
\newblock \showarticletitle{QuantumNAS: Noise-adaptive search for robust
  quantum circuits}.
\newblock \bibinfo{journal}{\emph{HPCA}} (\bibinfo{year}{2022}).
\newblock


\bibitem[\protect\citeauthoryear{Wang, Li, Gu, et~al\mbox{.}}{Wang
  et~al\mbox{.}}{2022b}]%
        {wang2022onchipqnn}
\bibfield{author}{\bibinfo{person}{Hanrui Wang}, \bibinfo{person}{Zirui Li},
  \bibinfo{person}{Jiaqi Gu}, {et~al\mbox{.}}}
  \bibinfo{year}{2022}\natexlab{b}.
\newblock \showarticletitle{QOC: Quantum On-Chip Training with Parameter Shift
  and Gradient Pruning}.
\newblock \bibinfo{journal}{\emph{DAC}} (\bibinfo{year}{2022}).
\newblock


\bibitem[\protect\citeauthoryear{Wang, Zhang, and Han}{Wang
  et~al\mbox{.}}{2021b}]%
        {wang2021spatten}
\bibfield{author}{\bibinfo{person}{Hanrui Wang}, \bibinfo{person}{Zhekai
  Zhang}, {and} \bibinfo{person}{Song Han}.} \bibinfo{year}{2021}\natexlab{b}.
\newblock \showarticletitle{Spatten: Efficient sparse attention architecture
  with cascade token and head pruning}. In \bibinfo{booktitle}{\emph{HPCA}}.
  IEEE, \bibinfo{pages}{97--110}.
\newblock


\bibitem[\protect\citeauthoryear{Wang, Liang, Zhou, et~al\mbox{.}}{Wang
  et~al\mbox{.}}{2021a}]%
        {wang2021exploration}
\bibfield{author}{\bibinfo{person}{Zhepeng Wang}, \bibinfo{person}{Zhiding
  Liang}, \bibinfo{person}{Shanglin Zhou}, {et~al\mbox{.}}}
  \bibinfo{year}{2021}\natexlab{a}.
\newblock \showarticletitle{Exploration of Quantum Neural Architecture by
  Mixing Quantum Neuron Designs}. In \bibinfo{booktitle}{\emph{ICCAD}}. IEEE.
\newblock


\bibitem[\protect\citeauthoryear{Wille, Burgholzer, and Zulehner}{Wille
  et~al\mbox{.}}{2019}]%
        {wille2019mapping}
\bibfield{author}{\bibinfo{person}{Robert Wille}, \bibinfo{person}{Lukas
  Burgholzer}, {and} \bibinfo{person}{Alwin Zulehner}.}
  \bibinfo{year}{2019}\natexlab{}.
\newblock \showarticletitle{Mapping quantum circuits to IBM QX architectures
  using the minimal number of SWAP and H operations}. In
  \bibinfo{booktitle}{\emph{DAC}}. IEEE, \bibinfo{pages}{1--6}.
\newblock


\bibitem[\protect\citeauthoryear{Xiao, Rasul, and Vollgraf}{Xiao
  et~al\mbox{.}}{2017}]%
        {xiao2017fashion}
\bibfield{author}{\bibinfo{person}{Han Xiao}, \bibinfo{person}{Kashif Rasul},
  {and} \bibinfo{person}{Roland Vollgraf}.} \bibinfo{year}{2017}\natexlab{}.
\newblock \showarticletitle{Fashion-mnist: a novel image dataset for
  benchmarking machine learning algorithms}.
\newblock \bibinfo{journal}{\emph{arXiv 1708.07747}} (\bibinfo{year}{2017}).
\newblock


\end{thebibliography}
